\documentclass{article}

\usepackage{microtype}
\usepackage{graphicx}
\usepackage{subfigure}
\usepackage{hyperref}
\usepackage[accepted]{icml2023}

\usepackage{amsmath}
\usepackage{amssymb}
\usepackage{mathtools}
\usepackage{amsthm}

\usepackage{bbm}
\usepackage{bm}
\usepackage{xspace}

\theoremstyle{plain}
\newtheorem{theorem}{Theorem}[section]
\newtheorem{proposition}[theorem]{Proposition}
\newtheorem{lemma}[theorem]{Lemma}
\newtheorem{corollary}[theorem]{Corollary}
\theoremstyle{definition}
\newtheorem{definition}[theorem]{Definition}

\theoremstyle{remark}

\newcommand{\blackboxconstant}{24 \times 10^4 \log(N)^4}
\newcommand{\D}{\bar D}
\newcommand{\J}{\bar J}
\newcommand{\unexploredj}{\bar j}
\newcommand{\jbar}{\unexploredj}
\newcommand{\R}{\mathbb{R}}
\renewcommand{\P}{\mathbb{P}}
\newcommand{\Pjzero}{\P^{(j_0)}}
\newcommand{\Pjbar}{\P^{(\jbar)}}
\newcommand{\Pj}{\P^{(j)}}
\newcommand{\Ejzero}{\E^{(j_0)}}

\newcommand{\ti}{_{t,i}}
\newcommand{\Ti}{_{T,i}}
\newcommand{\tit}{_{t,i_t}}
\newcommand{\tj}{_{t,j}}

\newcommand{\tk}{_{t,k}}
\newcommand{\tkl}{_{t,k,l}}

\newcommand{\skl}{_{s,k,l}}
\newcommand{\tikl}{^{k,l}_{t,i}}

\newcommand{\hLoss}{\widehat{L}}

\newcommand{\hl}{\hat{\ell}}
\newcommand{\Nout}{N^{out}}
\newcommand{\Nin}{N^{in}}
\newcommand{\E}{\mathbb{E}}
\newcommand{\EE}[1]{\E\left[{#1}\right]}

\newcommand{\I}{\mathbb{I}}
\newcommand{\actiont}{{i_t}} % action of the algorithm at time t
\newcommand{\range}[2]{#1,\,\dots,\,#2}

\newcommand{\mainAlgorithm}{\textsc{Exp3-EX}\xspace}
\newcommand{\etat}{\eta_t}
\newcommand{\gammat}{\gamma_t}

\newcommand{\RI}{R^*_{I}}
\newcommand{\Rstar}{R^*}
\newcommand{\Qstar}{Q^*}
\newcommand{\QI}{Q^*_{I}}
\newcommand{\QIDelta}{Q^*_{I,\Delta}}
\DeclareMathOperator{\KL}{KL}
\DeclareMathOperator*{\argmax}{arg\,max}
\DeclareMathOperator*{\argmin}{arg\,min}

\newcommand{\bp}{\bm{p}}

\newcommand{\bpi}{\boldsymbol{\pi}}

\newcommand{\sumt}{\sum_{t\in[T]}}
\newcommand{\sumi}{\sum_{i\in[N]}}
\newcommand{\sumj}{\sum_{j\in[N]}}

\newcommand{\sumk}{\sum_{k\in[K]}}
\newcommand{\suml}{\sum_{l\in[L]}}
\newcommand{\sumJ}{\sum_{i\in \J\tkl}}
\newcommand{\sumJprime}{\sum_{i\in \J'\tkl}}
\newcommand{\sumkl}{\sum_{\substack {k\in[K] \\ l\in[L]}}}

\newcommand{\cN}{\mathcal{N}}

\newcommand{\todo}[1]{{\color{red} #1}}
\newcommand{\alex}[1]{{\color{blue}}}
\newcommand{\tomas}[1]{{\color{cyan}}}
\newcommand{\hide}[1]{{\color{orange}}}

\icmltitlerunning{Online Learning with Feedback Graphs: The True Shape of Regret}

\begin{document}

\twocolumn[
\icmltitle{Online Learning with Feedback Graphs: The True Shape of Regret}

\icmlsetsymbol{equal}{*}

\begin{icmlauthorlist}
\icmlauthor{Tomáš Kocák}{up}
\icmlauthor{Alexandra Carpentier}{up}
\end{icmlauthorlist}

\icmlaffiliation{up}{Institute of Mathematics, University of Potsdam, Germany}

\icmlcorrespondingauthor{Tomáš Kocák}{kocak@math.uni-potsdam.de}

\icmlkeywords{Machine Learning, ICML}

\vskip 0.3in
]

\printAffiliationsAndNotice{}

\begin{abstract}
  Sequential learning with feedback graphs is a natural extension of the multi-armed bandit problem where the problem is equipped with an underlying graph structure that provides additional information - playing an action reveals the losses of all the neighbors of the action. This problem was introduced by \citet{mannor2011} and received considerable attention in recent years. It is generally stated in the literature that the minimax regret rate for this problem is of order $\sqrt{\alpha T}$, where $\alpha$ is the independence number of the graph, and $T$ is the time horizon. However, this is proven only when the number of rounds $T$ is larger than $\alpha^3$, which poses a significant restriction for the usability of this result in large graphs. In this paper, we define a new quantity $\Rstar$, called the \emph{problem complexity}, and prove that the minimax regret is proportional to $\Rstar$ for any graph and time horizon $T$. Introducing an intricate exploration strategy, we define the \mainAlgorithm algorithm that achieves the minimax optimal regret bound and becomes the first provably optimal algorithm for this setting, even if $T$ is smaller than $\alpha^3$.
\end{abstract}

\section{Introduction}

In this paper, we consider a sequential decision-making problem in an adversarial environment. This problem consists of $N$ actions, $T$ rounds, and sequence of losses $(\ell\ti)_{(t,i)\in[T]\times[N]}$ where $[K]\triangleq\{1,\dots,K\}$. Each loss $\ell\ti$ is associated with round $t$ and action $i$. We do not impose any statistical assumptions on the losses provided by the environment. Instead, we assume that the losses are set by an oblivious adversary before the learning process begins. The only assumption on the losses is that they are bounded in $[0,1]$, otherwise, the losses can be completely arbitrary and change in every round.

The learning process, or the game, proceeds in rounds. In round $t$, the learner picks one of the actions denoted by $i_t$ and incurs associated loss $\ell\tit$. The learner also observes loss $\ell\tit$ itself and possibly, losses of some other actions. The set of observations depends on the feedback scheme, we discuss different feedback schemes later.

The goal of the learner is to minimize the total loss received at the end of the game, after $T$ rounds. This is equivalent to minimizing the difference between the total loss of the learner and the loss of the strategy that plays the best-fixed action in hindsight, after $T$ rounds. We refer to this difference as regret and define it as

\[
  R_T \triangleq \max_{i\in[N]} \E\bigg[ \sumt (\ell_{t,\actiont} - \ell\ti)\bigg],
\]

where the expectation is taken over the potential randomization of the environment and the learner.

A quantity of interest to characterize the difficulty of such a sequential decision-making problem is what we will refer to here as the minimax regret, namely the regret incurred by the best possible strategy - the choice of actions $(i_t)_t$ - on the most difficult possible bandit problem - the choice of loss sequences $(\ell_{t,i})_{t,i}$. Note that the minimax regret depends on the feedback scheme considered.

Traditionally, this problem is studied under different feedback schemes. The most relevant schemes for our paper are the following:

\paragraph*{Full-information feedback} \cite{cesa-bianchi1997,littlestone1994,vov1990}, sometimes called prediction with expert advice. This feedback is the simplest since the learner has access to all losses. At the end of round $t$ the learner observes whole loss vector $(\ell_{t,1}\,\dots, \ell_{t,N})$. The minimax rate for this feedback scheme is $\sqrt{T\log(N)}$ and is attained by the EXP algorithm \citep{cesa-bianchi2006}. Note that having access to all the losses in every round causes the minimax rate scale only as $\sqrt{\log(N)}$ with the number of actions. 

\paragraph*{Bandit feedback} \citep{thompson1933,robbins1952,auer1995}. In every round, the learner observes only the loss of the selected action, namely $\ell_{t, i_t}$, while the losses of other actions are not disclosed. The minimax rate for this feedback scheme is $\sqrt{NT}$ \citep{audibert2010} and is attained by \textsc{INF} (Implicitly Normalized Forecaster) algorithm by \citet{audibert2010}. Having only one observation per round results in a scaling of the regret with $\sqrt{N}$ which is significantly worse than in the full-information feedback.

\paragraph*{Graph feedback} \citep{mannor2011,alon2013,alon2015,alon2017,kocak2014,kocak2016,kocak2016b,esposito2022}. In the graph feedback setting, the actions are vertices of a graph and in every round, the learner observes the loss of the selected action (so that the setting is strongly observable) as well as the losses of all its neighbors - see Section \ref{sec:setting} for a precise definition. This is the feedback scheme that we consider in this paper, which is an intermediary between full-information and bandit feedback and contains both these settings. Similarly to what happens in the bandit setting, the algorithms for bandits with graph feedback need to balance \emph{exploration} of actions with \emph{exploitation} of already acquired knowledge. In the graph feedback setting, however, different actions might provide different amounts of exploration, as an action also provides information on the losses of its neighbors. So that balancing exploration and exploitation in this context is more delicate, and efficient algorithms will need to adapt to the graph structure -  and the minimax regret will also be graph dependent. In this setting, a relevant graph-dependent quantity is the independence number $\alpha$ of the graph (see Definition \ref{def:independence-number}). Several algorithms with different approaches have been proposed, \textsc{ELP} \cite{mannor2011}, \textsc{Exp3-SET} and \textsc{Exp3-DOM} \citep{alon2013}, \textsc{Exp3-IX} and \textsc{FPL-IX} \citep{kocak2014}, \textsc{Exp3.G} \citep{alon2015}. While these algorithms differ in their approach to exploration, assumptions on the graph disclosure, or computational complexity, the common denominator is that all of these algorithms' upper bounds on the regret, in the case of strongly observable graphs, are of order $\sqrt{\alpha T}$ up to logarithmic terms, regardless of time horizon $T$. All of the aforementioned algorithms were inspired by the lower bound for the setting proposed by \citet{mannor2011}, which states that if  $T\ge 374\alpha^3$, the minimax regret is lower bounded by a quantity of the order $\sqrt{\alpha T}$ - see Proposition \ref{prop:mannor-lower-bound} for a precise quotation of their result. This poses the question of what happens for large graphs - or equivalently when $T$ is small - and whether current algorithms are also optimal in this case. This is a very important question since even in a moderately large problem and for some graphs, where the independence number is in the hundreds, we need to have millions of rounds for this assumption to hold.

\paragraph*{Partial monitoring.}
A bit further from the setting that we consider in this paper, yet related to it, is the field of partial monitoring \citep{rustichini1999,audibert2010,lattimore2019,lattimore2020}, where the action selection is decoupled from the feedback. An example of this which is very relevant for us is weakly observed graphs - see~\citep{alon2015} - which is a generalization of the graph feedback setting where not all self-loops are included, which means that one does not necessarily observes the loss of the action that one selects. The algorithm \textsc{Exp3.G} therein takes advantage of small dominating sets of vertices - i.e.~sets of vertices whose joint set of neighbors are all vertices, see Section \ref{sec:setting} for a precise definition - to explore efficiently the vertices, and then focus on promising actions. While not developed for the setting considered in this paper, this approach opens however interesting perspectives in cases of large graphs with a few very connected vertices, and we will discuss this more in detail in Subsection~\ref{sec:motivational-example}.

\subsection{Contribution}

In this paper, we focus on the setting with graph feedback - see Section \ref{sec:setting} - and our aim is to pinpoint the minimax regret in the missing case presented in the corresponding paragraph above, namely for large graphs where $T$ is of smaller order than $\alpha^3$.

The first important remark that we make in this paper is that there are some simple cases of large graphs where it is possible to achieve a minimax regret of much smaller order than $\sqrt{\alpha T}$, which is the current best known upper bound. This is e.g.~the case when there is one action that is connected in the graph to all other actions, and that is therefore very informative. In this case, if $T$ is of smaller order than $\alpha^3$, a minimax optimal strategy would make heavy use of this action in order to explore the other actions, even if this action is sub-optimal. We detail such an example in \ref{sec:motivational-example}. This is very different from what current algorithms in the (strongly observable) graph bandit literature do and is more related to some strategies in partial monitoring, see e.g.~\citep{lattimore2019} and also~\citep{alon2015} that we will discuss in detail later. Starting from this remark, the main result of this paper is to pinpoint, for any time horizon $T$ and any given graph, the minimax regret up to logarithmic terms. We first provide a more refined lower bound in Section \ref{sec:lower-bound}, that holds for any graph and time horizon - therefore also in the case where $T< 374\alpha^3$ which is not covered by the state of the art lower bound in \citet{mannor2011} - and that involves a more subtle graph dependent quantity than the independence number. Then, in Section \ref{sec:algorithm}, we provide \mainAlgorithm algorithm (EX stands for \textbf{E}xplicit e\textbf{X}ploration) that matches this lower bound up to logarithmic terms, and whose particularity is that it explores informative actions in a refined and explicit way.

\section{Problem Setting}
\label{sec:setting}

In this section, we formally define the setting introduced by \citet{mannor2011} and provide all the notation used throughout the paper.

We consider an online learning game with a directed observability graph $G=(V,E)$ over the set of actions $V=[N]$ with the set of edges $E \subseteq [N]\times[N]$. The graph contains all the self-loops, i.e. $(i,i)\in E$ for every $i\in V$. The indicator function of an edge from node $i$ to node $j$ is defined as $G_{i,j}\triangleq\I\{(i,j)\in E\}$. The game takes place over $T$ rounds. Before the game starts the environment, potentially adversarial, assigns losses $\{\ell\ti\}_{(t, i)\in[T]\times[N]}$ to every action $i$ and round $t$. We only assume that $\ell\ti \in [0,1]$ for any $t \leq T, i \leq N$.

In every round $t$, the learner picks an action $\actiont\in[N]$, incurs the loss $\ell_{t,\actiont}$, and observes the losses $\ell\ti$ of all out neighbors of $\actiont$, i.e.~of all $i\in V$ such that $(\actiont,i)\in E$\footnote{We write $\Nout_\actiont$ for this set, see definition \ref{def:neighborhood} later.}. Note that in our setting, we always observe the loss of the chosen action since the graph contains all the self-loops. The performance of the learner is then measured in terms of regret - sometimes also called pseudo-regret, or also expected regret - as explained in the introduction% sometimes also called pseudo-regret, defined as

\[
  R_T \triangleq \max_{i\in[N]} \E\bigg[ \sumt (\ell_{t,\actiont} - \ell\ti)\bigg],
\]

where the expectation is taken over the potential randomization of the environment and the learner.

\subsection{Auxiliary Definitions and Statements}

This section sums up all the necessary graph-related definitions we use later throughout the paper.

In bandits with graph feedback, the learner's task is to select an action and observe the losses of its neighbors. Each loss observation can have different sources, either the learner selected the action itself or one of its neighbors. The following definition provides us with a tool to define side observations and their sources more easily.

\begin{definition}\label{def:neighborhood}
  Let $G=(V,E)$ be a graph with the set of vertices $V$ and the set of edges $E$. We define the out-neighborhood of vertex $i\in V$ as
  
\[
    \Nout_i \triangleq \left\{ j \in V \,:\, (i,j)\in E\right\}
\]

and the in-neighborhood of vertex $i\in V$ as

\[
    \Nin_i \triangleq \left\{ j \in V \,:\, (j,i)\in E\right\}
\]

\end{definition}

Playing only a few actions can provide the learner with information about many other actions. Dominating sets and numbers provide a convenient way to describe this phenomenon.

\begin{definition}\label{def:dominating-number}
  Let $G=(V, E)$ be a graph with the set of vertices $V$ and the set of edges $E$. We say that $D\subseteq V$ is a dominating set of $B\subseteq V$ (or that $D$ dominates $B$) from $A\subseteq V$ if $D\subseteq A$ and $B\subseteq \cup_{i\in D} \Nout_i$. We define the dominating number of $B$ from $A$ as  $\delta^A(B) \triangleq \min |D|$ where the minimum is taken over all dominating sets $D$ of $B$ from $A$. In case no such $D$ exists, we define $\delta^A(B)$ as $\infty$. Further, we say that $\delta(B)\triangleq\delta^V(B)$ is the dominating number of $B$ and $\delta \triangleq \delta^V(V)$ is the dominating number of the graph.
\end{definition}

On the other hand, if the actions are not connected by an edge, playing one action does not provide any additional information about the other actions. This is captured in the following definition of independent sets.  

\begin{definition}\label{def:independence-number}
  Let $G=(V,E)$ be a graph with the set of vertices $V$ and the set of edges $E$. We say that $I\subseteq V$ is an independent set of $G$ if for every $i, j\in I$ s.t. $i\not=j$, vertices $i$ and $j$ are not connected by an edge, i.e. $(i,j)\not\in E$. Independence number $\alpha$ of $G$ is the size of the largest independent set of $G$, i.e.

\[
  \alpha \triangleq \max_{I\in\{J\subseteq V\, :\, J \textrm{ is independent}\}} |I|.
\]

\end{definition}

%%%%%%%%%%%%%%%%%%%%%%%%%%%%%%%%%%%%%%%%%%
%%%%%%%%%% Original lower bound %%%%%%%%%%
%%%%%%%%%%%%%%%%%%%%%%%%%%%%%%%%%%%%%%%%%%

\subsection{Lower Bound by \citet{mannor2011} and Motivational Example}\label{sec:motivational-example}

In this subsection, we quote formally an important and state-of-the-art result of the literature and discuss why some relevant graph feedback examples are not optimally resolved by existing algorithms.

The following proposition restates the lower bound result by \citet{mannor2011}.

\begin{proposition}
  \label{prop:mannor-lower-bound}
  Let $G$ be an observability graph with independence number $\alpha$. Then there exists a series of losses $\{\ell\ti\}_{(t,i)\in[T]\times[N]}$ such that for every $T\ge 374\alpha^3$ and any learner, the expected regret is at least $0.06\sqrt{\alpha T}$
\end{proposition}
It is important to note that the statement assumes that $T$ needs to be large  - $T\ge 374\alpha^3$ - for the lower bound to hold. As mentioned in the introduction, some existing algorithms match this lower bound up to logarithmic factors \citep[Corollary 1]{kocak2014} \citep[Theorem 1]{alon2015}. These algorithms also function when $T< 374\alpha^3$ where the best known upper bounds on the regret are of order $\sqrt{\alpha T}$ up to logarithmic factors. However, since the existing lower bound stated above does not cover this case, it is therefore unclear whether those algorithms are optimal or not.

%%%%%%%%%%%%%%%%%%%%%%%%%%%%%%%%%%%%%%%%%%
%%%%%%%%%% Motivational example %%%%%%%%%%
%%%%%%%%%%%%%%%%%%%%%%%%%%%%%%%%%%%%%%%%%%

The following lemma demonstrates that $\sqrt{\alpha T}$ is indeed not the correct rate.

\begin{lemma}
\label{lem:star-graph-example}
    Let $G = (V,E)$ be a graph with $|V| = N$ and $E = \{(N,i): i\in[N-1]\}$ (see Figure~\ref{fig:motivational-example}). Then, there exists an algorithm such that the regret upper bound of this algorithm is of $\delta^{1/3}T^{2/3}$ where $\delta = 1$ is the dominating number of $G$. 
\end{lemma}

The independence number of the graph from the previous lemma is $N-1$. This also means that whenever $T\ll\alpha^3$ in the lemma above, the regret bound of $\delta^{1/3}T^{2/3}$ is an improvement over the regret bound of $\sqrt{\alpha T}$. See Appendix~\ref{sec:appendix-setting} for the proof and further discussion.

\begin{figure}[t]
\center
  \includegraphics[width=.8\linewidth]{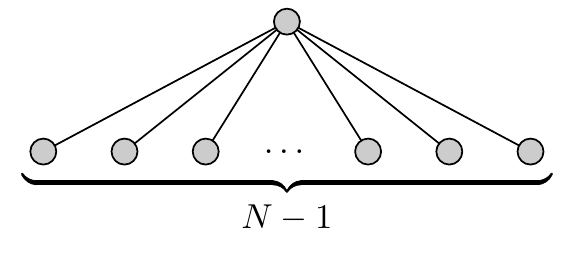}
  \caption{Bandit problem with one hub action that observes all other $N-1$ actions.}
  \label{fig:motivational-example}
\end{figure}

%%%%%%%%%%%%%%%%%%%%%%%%%%%%%%%%%%%%%%%%
%%%%%%%%%% Problem complexity %%%%%%%%%%
%%%%%%%%%%%%%%%%%%%%%%%%%%%%%%%%%%%%%%%%

\subsection{Problem Complexity}\label{ss:complexity}

We have seen some indications (e.g. the example in Lemma~\ref{lem:star-graph-example}) that for small $T$, the minimax regret might not scale with $\sqrt{\alpha T}$.

In what follows, we define the graph-dependent problem complexity that will later appear in our lower and upper bounds. This quantity is complex and depends on the graph in a refined way. It is however what one would expect for a worst-case stochastic problem~-~namely a problem where losses $\ell_{i,t}$ are independent and sampled according to a distribution depending on $i$. In order to introduce the problem complexity, we resort to intuitions from the stochastic setting, although we will analyze the problem in an adversarial setting, and provide precise results later, see Theorems~\ref{thm:lower-bound} and~\ref{thm:upper-bound}.

Assume that we are given a set containing all promising actions that could be optimal, given available information - let us call it $I$, in a stochastic setting it would typically be a set of actions whose empirical mean confidence intervals intersect one of the actions with higher lower confidence bound. Let us oversimplify the problem and assume in this informal explication that all these actions but one have a small gap $\Delta>0$ with respect to the optimal action. The optimal action is also in $I$ and has a gap of $0$. When playing an action in $I$, one incurs an average instantaneous regret of $\Delta$ except if one samples the optimal action. We are facing the following choice when we try to find the optimal action: we can either sample in $I$ directly and have small instantaneous regret - namely $\Delta$ - or we can sample outside of this set and have an instantaneous regret that is in all generality bounded by $1$. However, sampling outside of $I$ might still be interesting if some of the actions there are connected to many actions in $I$, providing in this way very informative feedback on many actions therein - see e.g.~Figure~\ref{fig:motivational-example} where even if the hub action is clearly sub-optimal from the samples, it might still be interesting to take advantage of it. At the end of the budget and if one wants to have found the optimal action - and to therefore not pay an instantaneous regret of at least $\Delta$ at each round - one would need as is usual in stochastic bandits to have observed all actions - from inside or outside of $I$- at least $1/\Delta^2$ times. In the stochastic setting, we would therefore expect that the most difficult graph bandit problems would correspond to the worst choice of $(I,\Delta)$.

These considerations drive us to the first definition of the \textit{problem complexity $\Qstar$}.
\begin{definition}\label{def:problem-complexity-Q}
  Let $G=(V,E)$ be a graph and $T$ be the number of rounds. Then the problem complexity $\QI$ for given set $I\subseteq V$ is defined as

  \[
    \QI \triangleq \max_{\Delta\in(0,1/2]} \QIDelta
  \]

  where

  \[
    \QIDelta \triangleq \min_{\bpi\in\Pi}\min\Bigg[ T\sum_{i\in I} \pi_i \Delta + T\sum_{i \not\in I} \pi_i,\, T\Delta \Bigg]
  \]

  and 

  \begin{align*}
    \Pi & = \bigg\{\bpi\in \mathbb R_{+}^{N}:\sumi \pi_i \le1, \\  &\qquad\qquad\qquad T\sumi \pi_i G_{i, j} \geq 1/\Delta^2, \forall j \in I\bigg\}.
  \end{align*}

  Moreover, we define the problem complexity $\Qstar$ as

  \[
    \Qstar \triangleq \max_{I\subseteq V} \QI.
  \]

\end{definition}

$\QIDelta$ would correspond to the regret of the best stationary policy $\pi$ over a problem as described above, for a fixed set $I$ and gap $\Delta$. The worst-case problem is then obtained by taking the worst case of set $I$ and gap $\Delta$.

Unfortunately, the quantity defined in Definition~\ref{def:problem-complexity-Q} is very unintuitive, in that it is unclear how it relates to quantities such as dominating numbers, and independence numbers, of (sub-)graphs. we, therefore, define another relevant notion of the problem complexity $\Rstar$.

\hide{
As we discussed earlier, there is a trade-off between sampling in $I$ and sampling outside of $I$ - and we would expect that an optimal algorithm, would allocate its samples in the best possible way for solving this trade-off, as would e.g.~an optimal stationary policy $\bpi$ that would achieve $\QIDelta$. So that we expect that for an optimal algorithm, there will be two types of actions in $I$: actions in a set $J \subset I$ that are better explored "from inside" - namely either by sampling them directly or by sampling other actions in $I$ that are connected to them - and actions in $I \setminus J$ that are better sampled "from outside" - namely by sampling some actions outside of $I$ that are connected to them. For the best choice of $J$, an intuitive choice would be, to respectively sample a dominating set of $J$ from $I$ for exploring actions in $J$, and a dominating set of $I \setminus J$ for exploring actions in $I \setminus J$. With that in mind, and by taking explicitly a worst case of $I$ and implicitly a worst-case $\Delta$ - remembering that actions in $I$ are $\Delta$-sub-optimal while actions outside of $I$ are $1$-sub-optimal - we, therefore, define another relevant notion of the problem complexity $\Rstar$.
}

\begin{definition}\label{def:problem-complexity-R}
  Let $G=(V,E)$ be a graph and $T$ be a number of rounds. Then the problem complexity $\RI$ for given set $I\subseteq V$ is defined as

  \[
    \RI \triangleq \min_{J\subseteq I}\max\left\{ \delta^I(J)^\frac{1}{2}T^\frac{1}{2},\, \delta^V(I\setminus J)^\frac{1}{3}T^\frac{2}{3} \right\}.
  \]

  Moreover, we define the problem complexity $\Rstar$ as

  \[
    \Rstar \triangleq \max_{I\subseteq V} \RI.
  \]

\end{definition}

This definition is much more tractable from a graph perspective, as it involves only two relevant graph-dependent quantities, namely the dominating set $\delta^I(J)$ of $J$ from $I$, and the dominating number $\delta^V(I\setminus J)$ of $I\setminus J$ from $V$. Here the choice of the optimal policy is reduced to only choosing the set $J$ that is best explored from inside of $I$, and we then select the worst case of $I$.

Interestingly, the following lemma shows that both definitions of the problem complexity are almost equivalent and differ only up to a logarithmic factor. From now on, whenever talking about the problem complexity, we specify which definition we use and the reasons why.
\begin{lemma}\label{lem:complexity-equivalence}
  Let $G=(V, E)$ be a graph and $I\subseteq V$ be any set of actions. Then for the problem complexities $\QI$ and $\RI$, the following inequalities hold. 
  \[
    \RI / (10\log N) \le \QI \le 2 \RI
  \]
\end{lemma}
The proof of this lemma can be found in Appendix~\ref{sec:problem-complexity-appendix}.
\section{Lower Bound}
\label{sec:lower-bound}

Ever since the introduction of the setting by \citet{mannor2011}, the lower bound in Proposition \ref{prop:mannor-lower-bound} was used to drive the ideas for the algorithms. However, in general, this lower bound holds only when $T\ge374\alpha^3$.

Even though approaches of algorithms and their upper bound analyses differ from paper to paper, most of them are able to match the lower bound for $T\ge374\alpha^3$. Without the lower bound for regimes where $T<374\alpha^3$, there was no incentive for the algorithms to strive for a different rate than the one suggested by Proposition \ref{prop:mannor-lower-bound}. In this section, we present a new lower bound that holds regardless of the value of $T$ and thus, extend the result in Proposition \ref{prop:mannor-lower-bound}. 

The following theorem shows a regret lower bound that scales with the problem complexity $\Rstar$ and is one of the main results of our paper.

\begin{theorem}
  \label{thm:lower-bound}
  Let $G=(V, E)$ be a directed graph with $N = |V|$ and $T$ be a number of rounds. Then, for any learner, there exists a sequence of randomized losses such that regret $R_T$ of the learner is lower bounded as
  \begin{equation*}
    R_T \ge \frac{\Qstar}{2^7}
    \ge \frac{\Rstar}{2^710\log N}.
  \end{equation*}
  where $\Qstar$ and $\Rstar$ are problem complexities  
\end{theorem}

\textit{Proof idea.} 
The idea of the proof follows standard lower bound proof steps, see e.g. \citep[Chapter 15]{lattimore2020}, with our problem-specific twist. The idea is to create a set of "difficult" stochastic bandit problems and show that no matter what the learner does, there always will be at least two different problems that the learner can not distinguish.

We create the problems by first choosing a set of near-optimal actions $I$ and then setting the gap of every action outside of $I$ to 1 and inside of $I$ to some small constant $\Delta$. The only exception is the optimal action. For different problems, we choose different optimal actions from $I$ and set its gap to 0.

Using information-theoretic tools, we can show that every action needs to be explored enough, i.e. at least $1/\Delta^2$ times, in order to be able to distinguish the problems. The result of the theorem is then obtained by carefully choosing the gap parameter $\Delta$ and the set of difficult actions $I$.

The detailed proof of the theorem can be found in Appendix~\ref{sec:lower-bound-proof}. Note that the lower bound in Theorem~\ref{thm:lower-bound} scales with either of the definitions of the problem complexity. Later, we show that the rate depending on the problem complexity is indeed minimax optimal and we comment on the connection to the rate in the previous papers in Section~\ref{sec:discussion}.
\section{Algorithm}
\label{sec:algorithm}

The algorithm for our setting, similarly to the previous papers, uses exponential weights to define a probability distribution over the set of actions and then samples according to this distribution. Similarly to \textsc{Exp3.G} algorithm by \citet{alon2015}, we add extra exploration to some actions. This extra exploration adapts to the estimated quality of each action, but also to its informativeness, i.e.~to how much it is connected to other promising actions on the graph. 

The construction of the exploration distribution is rather intricate and is the main algorithmic contribution of this paper, and we will discuss it in detail later. We first present the main algorithm, with part devoted to the construction of this exploration distribution. We then present associated theoretical guarantees.

%%%%%%%%%%%%%%%%%%%%%%%%%%%%%%%%%%%%%%%%
%%%%%%%%%%%%%%%%%%%%%%%%%%%%%%%%%%%%%%%%
\subsection{Main Algorithm}
\label{sec:main-algorithm}
%%%%%%%%%%%%%%%%%%%%%%%%%%%%%%%%%%%%%%%%
%%%%%%%%%%%%%%%%%%%%%%%%%%%%%%%%%%%%%%%%

As is usual in the literature, our algorithm (\mainAlgorithm presented in Algorithm~\ref{alg:main}) updates at each time $t$ a probability distribution $\mathbf{p}_t$. It then plays at each round action $i_t$, which in the graph setting reveals the losses $\ell\ti$ of all its neighbors $i\in \Nout_{i_t}$. These losses enable us to update unbiased estimates of the (cumulative) loss estimates, which will be used in the algorithm.

\paragraph{(Cumulative) Loss estimates.} In the graph setting, the probability $P\ti$ of observing loss $\ell\ti$ is simply the sum of probabilities of playing any of the in-neighbors of arm $i$ and is defined as $P\ti \triangleq \sum_{j\in\Nin_i}p\tj$. This allows us, at the end of each round $t$, to construct, conditionally unbiased loss estimates $\hl\ti$ of the loss of each action

\[
  \hl\ti \triangleq \frac{\ell\ti \I\{i\in\Nout_{i_t}\} }{ P\ti} \qquad \textrm{for all}\qquad i\in[N]
\]

and to also update the cumulative loss estimates as

\[
  \hLoss\ti \triangleq \hLoss_{t-1,i} + \hl\ti \qquad \textrm{for all}\qquad i\in[N].
\]

We now describe the construction of the distribution $\mathbf{p}_t$. As in~\citep{alon2015} and in several other papers from the literature, we mix a distribution based on exponential weights - i.e.~we update at each time $t$ the exponential weights $\mathbf{w}_t$ and define a normalized distribution $\mathbf{q}_t$, as in~\citep{auer2002} - with an exploration distribution $\mathbf{u}_t$. We postpone the construction of the learning distribution to  Section~\ref{sec:mixing-and-partitioning} (summary in Definition~\ref{def:exploration-distribution}), as it is intricate and our main algorithmic contribution. We describe below how we construct $\mathbf{w}_t, \mathbf{q}_t, \mathbf{p}_t$, based on $\mathbf{u}_t$.

\paragraph{Recallibration of the parameters.} Our algorithm first recalibrates at every step the learning rate $\etat$ of the EXP3 part of the algorithm - we describe later in \alex{ref} how it is chosen. Our algorithm uses it to also callibrates the mixing probability $\gammat \triangleq \min\{(\etat T)^{-1}, 1/2\}$ of sampling the exploration distribution $\mathbf{u}_t$.

\paragraph{(Renormalized) Exponential weights.} Based on the cumulative loss estimate $\hLoss_{t-1,i}$ of arm $i$ at time $t$, we can define as in~\cite{auer2002} the exponential weights $w\ti$ as

\[
  w\ti \triangleq \exp(-\etat\hLoss_{t-1,i}) \qquad \textrm{for all}\qquad i\in[N].
\]

 Using these weights, we can construct a distribution simply by re-normalizing them to define

\begin{equation}
  q\ti \triangleq \frac{w\ti}{W_t} \triangleq \frac{w\ti}{\sumj w\tj} \qquad \textrm{for all}\qquad i\in[N]. \label{eq:exp-distribution}
\end{equation}

\paragraph{Mixed distribution.} Based on $\mathbf{u}_t$, we define our sampling distribution as
\begin{equation}
  p\ti \triangleq (1-\gammat)q\ti + \gammat u\ti \qquad \textrm{for all}\qquad i\in[N].  \label{eq:mixed-distribution}
\end{equation}

So far the algorithm is not different from the majority of algorithms designed for the graph setting and it is summarized in Algorithm~\ref{alg:main}. The key difference lies in the exploration distributions $(u\ti)_{i\in[N]}$ leveraging the structure of the graph, and in the learning rates $\etat$ defined later in Theorem~\ref{thm:upper-bound}. Especially exploration distributions $(u\ti)_{i\in[N]}$ set our algorithm apart from the previous algorithms and enables us to improve the upper bound to match the newly proposed lower bound. The following section explains all the details necessary for the definition of the exploration distributions.

\begin{algorithm}[tb]
  \caption{\mainAlgorithm}
  \label{alg:main}
  \begin{algorithmic}
    \STATE {\bfseries Input:} $G = (V,E)$, $\hLoss_{0,i} = 0$ for all $i\in[N]$
    \FOR{$t=1$ {\bfseries to} $T$}
    \STATE Set learning rate $\etat$\hfill(see Theorem~\ref{thm:upper-bound})
    \STATE $\gammat = \min\{(\etat T)^{-1}, 1/2\}$\\
    \STATE $w\ti = (1/N)\exp(-\eta \hLoss_{t-1,i})$ \hfill
    \STATE $W_t = \sumi w\ti$
    \STATE $q\ti = w\ti / W_t$
    \STATE $p\ti = (1-\gammat)q\ti + \gammat u\ti$ \hfill(see Definition~\ref{def:exploration-distribution})
    \STATE Choose $\actiont \sim \bp_t = (\range{p_{t,1}}{p_{t,N}})$
    \STATE Observe losses $\ell\ti$ for $i \in \Nout(i_t)$
    \STATE $P\ti = \sum_{j\in \Nin(i)} p\tj$ \hfill
    \STATE $\hl\ti = \ell\ti \I\{i\in\Nout(i_t)\} / P\ti$
    \STATE $\hLoss\ti=\hLoss_{t-1,i} + \hl\ti$
    \ENDFOR
  \end{algorithmic}
\end{algorithm}

%%%%%%%%%%%%%%%%%%%%%%%%%%%%%%%%%%%%%%%%
%%%%%%%%%%%%%%%%%%%%%%%%%%%%%%%%%%%%%%%%
\subsection{Mixing Distribution and Exploration}
\label{sec:mixing-and-partitioning}
%%%%%%%%%%%%%%%%%%%%%%%%%%%%%%%%%%%%%%%%
%%%%%%%%%%%%%%%%%%%%%%%%%%%%%%%%%%%%%%%%

In the graph setting, the interest of an algorithm for sampling an arm is not only characterized by the quality of this arm - i.e.~minus its cumulative loss at time $t$ - but also by the informativeness of this algorithm on other relevant arms - namely, whether or not it is connected to many arms with small cumulative loss at time $t$. While a classical adversarial bandit algorithm would take into account the first of these two factors, we need to add extra exploration to take into account the second factor, namely the connections of the arms through the graph structure.

The idea of the algorithm is to homogenize the actions by grouping them up according to their cumulative loss as well as the amount of information they provide and then define the exploration for each partition separately. We create the partitioning in two steps.

\paragraph{Partitioning of the actions into sets $(I\tkl)_{\tkl}$.} For every round $t$, we create partitions $\{I\tk\}_{k\in[K+1]}$, for $K = \lceil5\log_2(N)\rceil$, such that the normalized exponential weights of arms, defined in Equation~\ref{eq:exp-distribution}, are similar within a partition. More precisely define
\[
I\tk \triangleq \left\{ i \in [N]: q\ti \in (2^{-k}, 2^{-k+1}]\right\}.
\]
The last partition $I_{t, K+1}$ contains the rest of the arms, i.e.
\[
I_{t, K+1} \triangleq \left\{i\in [N]: q\ti \le 2^{-K}\right\}.
\]
Note that for every action $i\in I_{t, K+1}$, $q\ti$ can be upper bounded by $1/N^5$.

We further subdivide each set $I\tk$ into subsets that are roughly homogeneous in terms of the numbers of neighbors in $I\tk$.  For every arm $i\in I\tk$, we define $\deg\tk(i)$ as the number of neighbors of $i$ within partition $I\tk$:
\[
  \deg\tk(i) \triangleq |\{j\in I\tk:(i,j)\in E\}|.
\]
For $L = \lceil\log_2(N)\rceil$, we will further subdivide each set $I\tk$ into subsets $\{I\tkl\}_{l\in[L]}$ that are roughly homogeneous in terms of numbers of neighbors in $I\tk$ - namely, arm $i\in I\tkl$ if and only if $\deg\tk(i) \in (N2^{-l}, N2^{-l+1}]$. The following definition summarizes the construction of the partitions.

\begin{definition}
\label{def:partitioning}
  Let $K \triangleq \lceil5\log_2(N)\rceil$ and $L \triangleq \lceil\log_2(N)\rceil$, then for every $(t,k,l)\in[T]\times[K]\times[L]$ we define
  \begin{align*}
    I\tkl \triangleq \Big\{ i\in[N]:\  & q\ti \in \big(2^{-k}, 2^{-k+1}\big], \\ &\deg\tk(i) \in \big(N2^{-l}, N2^{-l+1}\big] \Big\}
  \end{align*}
  and
  \[
    I_{t,K+1} \triangleq \Big\{ i\in[N]: q\ti \le 2^{-K}\Big\}.
  \]
\end{definition}

\paragraph{Partition in exploration sets from inside and outside of $I\tkl$.}

Inspired by the definition of the problem complexity in Definition~\ref{def:problem-complexity-R}, we can define a splitting of every set $I\tkl$, for $k\in[K]$ and $l\in[L]$, into two parts, $J\tkl$ and $J'\tkl \triangleq I\tkl \setminus J\tkl$ that minimize expression
\begin{equation*}
  \max\left\{ \delta^{I\tkl}(J\tkl)^\frac{1}{2}T^\frac{1}{2},\, \delta^V(J'\tkl)^\frac{1}{3}T^\frac{2}{3}\right\}.   
\end{equation*}
We write $\Rstar_{I\tkl}$ the value of this minimum for each given set $I\tkl$. As we discussed in Section~\ref{ss:complexity}, an optimal exploration of the set $J\tkl$ can be done using actions in $I\tkl$, while an optimal exploration of $J'\tkl$ can be performed using actions outside $I\tkl$.

In order to construct our exploration distribution, we would like to have access to the sets $J\tkl$ and $J'\tkl$, and more specifically to some (approximate) dominating sets, in order to be able to define the exploration distribution. While it is possible in theory to find these sets based on the $I\tkl$ and on the graph, solving the optimization problem that leads to them can be computationally very expensive.

For this reason, we do not work directly with the sets $J\tkl, J'\tkl$, but rather with some approximations that are computationally tractable. Such approximations exist, as stated in Corollary~\ref{cor:partition-splitting} below, and are described in Appendix~\ref{sec:convex-program}. 

\begin{corollary}
\label{cor:partition-splitting}
The algorithm described in Appendix~\ref{sec:convex-program}, which is polynomial time in $N$ (as it consists in solving a linear optimization problem under linear constraints) outputs partitions $\J\tkl$, $\J'\tkl$ of $I\tkl$ together with their corresponding dominating sets $\D\tkl$, $\D'\tkl$, which satisfy
\begin{align*}
|\D\tkl| &\le \log(N)\delta^{I\tkl}(\J\tkl),\\
|\D'\tkl| &\le \log(N)\delta^V(\J'\tkl),
\end{align*}
and
\begin{align*}
&\max\left\{ |\D\tkl|^\frac{1}{2}T^\frac{1}{2},\, |\D'\tkl|^\frac{1}{3}T^\frac{2}{3} \right\} \le \\&\qquad\qquad\qquad\qquad\le \blackboxconstant\sqrt{\log(N)}\Rstar_{I\tkl}.
\end{align*}
\end{corollary}
As mentioned, $\J\tkl$ (resp. $\J'\tkl$) serves as a surrogate of $J\tkl$ (resp. $J'\tkl$) and dominating set $\D\tkl$ (resp. $\D'\tkl$) is an approximation of the smallest dominating set of $\J\tkl$ (resp. $\J'\tkl$) from $I\tkl$ (resp.~$V$). While the full construction of these sets is deferred to Appendix~\ref{sec:convex-program}, we discuss and sketch briefly their construction in Subsection~\ref{ss:complexity}.

Having an efficient way of computing partitions $\J'\tkl$ and their dominating sets $\D'\tkl$ allows us to define the following exploration distribution

\begin{definition}
\label{def:exploration-distribution}
let $I\tkl$, for $(t,k,l)\in[T]\times[K]\times[L]$ be a partition of $V$ from Definition~\ref{def:partitioning} and $\D'\tkl$ be a dominating set, from Corollary~\ref{cor:partition-splitting}. Then, we can define,
\begin{equation}
u\ti \triangleq \frac{1}{{KL+1}}\left(\frac{1}{N}+\sumk\suml u\tikl\right) \label{eq:mixing-distribution}
\end{equation}
where
\begin{align*}
  u\tikl & = \frac{1}{|\D'\tkl|} \qquad & \textrm{for all}\qquad i \in \D'\tkl      \\
  u\tikl & = 0 \qquad                 & \textrm{for all}\qquad i \not\in \D'\tkl.
\end{align*}
\end{definition}

Distribution $(u\ti)_{i\in[N]}$ can be seen as a mixture of uniform distributions where the term $1/N$ in Equation~\ref{eq:mixing-distribution} corresponds to the uniform distribution over all the actions and $(u\tikl)_{i\in[N]}$ corresponds to the uniform distribution over set $\D'\tkl$ which, as a consequence, secures exploration of $\J'\tkl$.

\subsection{Main Upper Bound Theorem}
\label{sec:main-theorem}

Utilization of the exploration distributions $(u\ti)_{i\in[N]}$ from the previous section and appropriately tuned learning rates $\etat$ enable us to prove the optimal regret upper bound for Algorithm~\ref{alg:main} stated in the following theorem.

\begin{theorem}
  \label{thm:upper-bound}
Let learning rate $\etat$ is defined as
  \[
  \min_{s\in[t]}\min_{k\in[K]}\min_{l\in[L]}\min\left\{|\D\skl|^{-\frac{1}{2}} T^{-\frac{1}{2}} , |\D'\skl|^{-\frac{1}{3}} T^{-\frac{2}{3}} \right\},
  \]
  where we remind that $\D\tkl$ and $\D'\tkl$ be the dominating sets outputted by the algorithm described in Appendix~\ref{sec:convex-program}. Then the regret of Algorithm \ref{alg:main} is upper bounded as
  \[
    R_T \le 24\times 10^4 \log(N)^5 D \Rstar
  \]
  for
  \begin{align*}
    D &= 4KL + 2 + \big((KL)^2 + KL + 1\big)\log (N), \\
    K &= \lceil5\log_2(N)\rceil, \\
    L &=\lceil \log_2(N)\rceil.
  \end{align*}
\end{theorem}

\textit{Proof idea.}
The proof of the theorem relies heavily on the partitioning from the Definition~\ref{def:partitioning} by decomposing the regret along the partitions. Careful construction of partitions allows us to show that the actions corresponding to one individual partition contribute to regret with no more than $\Rstar$, up to logarithmic factors. The fact that the number of partitions $KL+1$ is only polylogarithmic in the number of actions allows us to obtain the final regret bound as the sum of regret bounds for individual partitions. 

The detailed proof of the theorem can be found in Appendix~\ref{sec:upper-bound-proof}. 

\section{Discussion}
\label{sec:discussion}

We have presented the regret lower bound for the setting (Section~\ref{sec:lower-bound}, Theorem~\ref{thm:lower-bound}) as well as the matching, up to logarithmic terms, regret upper bound for the proposed algorithm (Section~\ref{sec:algorithm}, Theorem~\ref{thm:upper-bound}). Together, these two theorems prove that the minimax rate for online learning with feedback graphs is proportional to the problem complexity $\Rstar$ (Definition~\ref{def:problem-complexity-R}). In this section, we compare the minimax rate presented in this paper to the previously known results and emphasize the improvements that we bring to the setting. We focus mainly on two regimes:

When $T$ is large enough when compared to $\alpha^3$, we recover results from the literature, namely a minimax rate of order $\sqrt{\alpha T}$ up to logarithmic terms. We do it by showing that the problem complexity $\Rstar$ is equal to $\sqrt{\alpha T}$ when $T$ is large enough.
    
When $T$ is small, we demonstrate the existence of graphs for which rate $\sqrt{\alpha T}$ is far from optimal. An important consequence of this statement is that all the algorithms proposed in the previous papers prove only suboptimal regret upper bounds for some graphs and budgets $T$. This also means that Algorithm~\ref{alg:main} is the first provably optimal algorithm for the setting in all possible problems and regimes.

%%%%%%%%%%%%%%%%%%%%%%%%%%%%%%%%%%%%%%%%
%%%%%%%%%%%%%%%%%%%%%%%%%%%%%%%%%%%%%%%%
\subsection{Regime when $T$ is Large}
\label{sec:large-T-regime}
%%%%%%%%%%%%%%%%%%%%%%%%%%%%%%%%%%%%%%%%
%%%%%%%%%%%%%%%%%%%%%%%%%%%%%%%%%%%%%%%%

Previous papers proved that the minimax regret scales with $\sqrt{\alpha T}$ whenever $T\ge 374\alpha^3$ while the minimax regret presented in this paper scales with the problem complexity $\Rstar$ instead. The following corollary shows that the two rates are up to log factors the same when $T$ is large enough.

\begin{corollary}
  \label{cor:rate-for-large-T}
  Let $G$ be a graph with independence number $\alpha$. Then for any $T\ge {\alpha^3}$, the problem complexity $\Rstar$ simplifies to
  \[
  \Rstar = \sqrt{\alpha T}.
  \]
\end{corollary}
The proof for this corollary can be found in Appendix~\ref{sec:proof-rate-for-large-T}. As our upper and lower bounds in Theorems~\ref{thm:lower-bound} and~\ref{thm:upper-bound} match $\Rstar$ up to logarithmic terms, we recover the results from the literature.

%%%%%%%%%%%%%%%%%%%%%%%%%%%%%%%%%%%%%%%%
%%%%%%%%%%%%%%%%%%%%%%%%%%%%%%%%%%%%%%%%
\subsection{Regime when $T$ is Small}
\label{sec:small-T-regime}
%%%%%%%%%%%%%%%%%%%%%%%%%%%%%%%%%%%%%%%%
%%%%%%%%%%%%%%%%%%%%%%%%%%%%%%%%%%%%%%%%

From the previous section, we know that the minimax rate for large enough $T$ is $\sqrt{\alpha T}$. It is also true that most of the prior algorithms can achieve a regret upper bound that scales with $\sqrt{\alpha T}$. However, at first glance, it is not obvious how significant the improvement of the newly defined problem complexity is. 

We return back to the example from Lemma~\ref{lem:star-graph-example} \hide{Section~\ref{sec:motivational-example}} and introduce a couple of examples demonstrating that rate $\sqrt{\alpha T}$ can be significantly sub-optimal.

\textbf{Example 1.} Lemma~\ref{lem:star-graph-example} provides an example where the graph contains $N-1$ independent vertices and one hub connected to all other vertices. The following Corollary states that the minimax rate for this graph indeed scales with $\delta^{1/3}T^{2/3}$ instead of $\sqrt{\alpha T}$.

\begin{corollary}
  \label{cor:rate-for-star-graph}
  Let $G=(V,E)$ be a graph on $N$ vertices with one hub, i.e. the set of edges is $E = \{(N,i) : i\in[N-1]\}$. Then for any $T < {\alpha^3}$, the problem complexity $\Rstar$ simplifies to
  \[
  \Rstar = T^\frac{2}{3}.
  \]
\end{corollary}
The proof for this corollary can be found in Appendix~\ref{sec:proof-rate-for-small-T}. This result also shows that with the increasing number of actions, the gap between $\sqrt{\alpha T}$ and the problem complexity can be arbitrarily large.

\textbf{Example 2.}
Generalizing the previous example, we can create a graph consisting of two parts. A star graph with $1+N_1$ vertices and $N_2$ independent vertices without any edges. Now the problem complexity is of order $(T^{2/3}+\sqrt{N_2 T}) \land \sqrt{(N_1+N_2)T}$ while $\alpha = N_1+N_2$ and $\delta = 1+N_2$. If either $N_2 \geq N_1$, or $T \geq N_1^3$, the problem complexities $R^*$ and $Q^*$ are of order $\sqrt{(N_1+N_2)T} = \sqrt{\alpha T}$ (up to logarithmic terms) as predicted by \citet{alon2015}. However, if $N_2 < N_1$ and $T < N_1^3$ (large star graph), then the problem complexity is of order $T^{2/3}$. This is an example where the minimax rate is much smaller than $\delta^{1/3} T^{2/3}$ or $\sqrt{\alpha T}$. This example also illustrates that the minimax rates from the previous papers are not valid when $T$ is small enough. In contrast, the true minimax rate scales with the problem complexity which demonstrates that it is important to adapt locally to the graph and global quantities like the dominating number, or the independence number, are not complex enough to describe the problem complexity.

\textbf{Example 3.}
Expanding the previous example, we consider a graph where we have $\sum_{k \leq K} (k+1) m_k$ vertices. This graph consists, for each $k \in \{1, \ldots, K\}$, of $m_k$ star graphs with $k+1$ vertices each with no connection to each other. In this case $\alpha = \sum_{k \leq K} k m_k$ and $\delta = \sum_{k \leq K} m_k$. Now, write $A$ for the set of indexes $k$ such that $\sqrt{m_k k T} \geq m_k^{1/3} T^{2/3}$. The problem complexities $Q^*, R^*$ are of order $\sqrt{T\sum_{k \not \in A} m_k k} + (\sum_{k\in A} m_k)^{1/3} T^{2/3})$. For a graph containing some large star graphs, e.g.~whenever $\sup_k m_k k^3 \geq T$, the rate is of order, up to logarithmic terms, of $(\sum_{k \in A} m_k)^{1/3} T^{2/3}$. This can be significantly smaller than $\delta^{1/3} T^{2/3}$ if $A$ is very different from $\{1, \ldots, n\}$, e.g.~when the graph contains a small number of very large star graphs and a moderate number of small star graphs - an extreme case being in the previous example.

These examples highlight that in the case of large graphs that are not homogeneous in the size of their hubs, the problem complexity is not driven by quantities like the dominating number or the independence number, but by some related quantities that are local in the graph. Our algorithm is able to adapt to such local structures.

\hide{
\paragraph{Discussion in the general case.} In the general case, it is of course more complicated, and the more involved to quantity $\Rstar$ drives the minimax regret. As explained in Subsection~\ref{ss:complexity} \alex{insert}, this quantity involves the worst possible set $I$ and the best partition of $I$ in $J$ and $I\setminus J$, the dominating number of $J$ from $I$ through $\sqrt{\delta^I(J)T}$, and the dominating number of $I\setminus J$ through $\delta^V(I\setminus J)^{1/3}T^{2/3}$. As explained in Subsection~\ref{ss:complexity} \alex{insert}, the idea behind this is that some actions in $I$ are best explored from within $I$ - namely those in $J$ - while the others are best explored from outside. The actions within $I$ are the most promising ones, the price to pay for them is lower than the price to pay for exploring from outside of $I$. While this idea was exposed in Subsection~\ref{ss:complexity} \alex{insert} \todo{this ref is probably wrong} in a stochastic setting which is not the one studied here,  Algorithm~\ref{alg:main} reaches this minimax rate in the adversarial setting. This algorithm constructs explicitly, for sets $I_{t,k,l}$ of promising actions with similar degrees constructed sequentially, proxies of associated optimal partitioning in sets that are best explored from inside and outside - namely $\bar J_{t,k,l}$ - and of associated dominating sets - namely $\bar D_{t,k,l}, \bar D'_{t,k,l}$. And it uses these sets to construct the exploration distribution. 
}

\subsection{Exploration Distribution}

We believe that the exploration distribution in Definition~\ref{def:exploration-distribution} plays a crucial role in adapting \textsc{Exp3} algorithm to the setting for small $T$ and that the algorithm is suboptimal without it. In general, exponential weights encourage playing actions with small cumulative loss but neglect actions that are highly informative, i.e. connected to many other actions. To correct this behavior, we look at every partition $I\tkl$ and identify the set $\J'\tkl$ from Corollary~\ref{cor:partition-splitting} and its dominating set $\D'\tkl$. We already know that the optimal way of exploring $\J'\tkl$ is by playing more informative actions in $\D'\tkl$. To enforce this behavior in the \textsc{Exp3} algorithm we simply add extra uniform exploration to actions in~$\D'\tkl$.

\hide{
\subsection{Construction of the Proxies $\bar D_{t,k,l}, \bar D'_{t,k,l}$}

The details of the construction of $\D\tkl$ and $\D'\tkl$ are postponed to Appendix~\ref{sec:convex-program}, however, the procedure is very intuitive. We start with any partition $I$. The optimization problem $\QIDelta$ is convex which allows us to solve it efficiently with any precision. By solving $\QIDelta$, we recover the optimal static policy $\bpi$ for an algorithm to play, i.e. action $i$ should be played $T\pi_i$ times. This allows us to find the actions that are mostly observed from outside of $I$ and thus create set $\J'$. To create a dominating set $\D'$ of $\J'$, we use a simple greedy strategy. Start with an empty set $\D'$, add one vertex that dominates the most vertices of $\J'$, remove dominated vertices from $\J'$, and repeat until every vertex of original $\J'$ is dominated. This produces a dominating set with the smallest size, up to a logarithmic factor.
}

\section*{Acknowledgements}
The work of A. Carpentier is partially supported by the Deutsche Forschungsgemeinschaft (DFG) Emmy Noether grant MuSyAD (CA 1488/1-1), by the DFG CRC 1294 'Data Assimilation', Project A03, by the DFG Forschungsgruppe FOR 5381 "Mathematical Statistics in the Information Age - Statistical Efficiency and Computational Tractability", Project TP 02, by the Agence Nationale de la Recherche (ANR) and the DFG on the French-German PRCI ANR ASCAI CA 1488/4-1 "Aktive und Batch-Segmentierung, Clustering und Seriation: Grundlagen der KI".

\bibliography{paper}
\bibliographystyle{icml2023}

%%%%% APPENDIX %%%%%

\newpage
\appendix
\onecolumn

\section{Suboptimality of $\sqrt{\alpha T}$ for some Graphs and Proof of Lemma~\ref{lem:star-graph-example}}
\label{sec:appendix-setting}

Consider a graph with $N-1$ independent vertices and one hub vertex connected to all other vertices - see Figure \ref{fig:motivational-example}. The independence number $\alpha$ of this graph is $N-1$. This also means that the regret bounds of existing algorithms scale with $\sqrt{TN}$. This rate is also attained by classical minimax bandit algorithms that do not take the graph structure into account at all. This also means that there is no theoretical advantage to using existing graph-based algorithms. 

However, there are several indications that playing the hub vertex might be a good idea. Especially, when $T$ is small and the graph is large, i.e. $N$ is large.

The first indication comes from the stochastic bandits where even simple \emph{Explore then Commit} algorithm that samples the hub vertex $T^{2/3}$ times and then commits to the empirically best action achieves regret of $T^{2/3}$ \citep[Theorem 6.1.]{lattimore2020}.

Another indication comes from the \textsc{Exp3.G} algorithm by \citet{alon2015}. This algorithm work in a slightly more general setting where the feedback graph can be weakly observable, i.e. the learner does not necessarily knows the loss of the selected actions. We can transform our example from Figure~\ref{fig:motivational-example} to the weakly observable case by not revealing the loss of the selected action to the learner whenever the learner selects a non-hub action. This makes the problem more difficult. Applying the \textsc{Exp3.G} algorithm would result in the regret bound of $\delta^{1/3}T^{2/3}$ where $\delta$ is the dominating number of the graph, in our case $\delta = 1$.

\section{Proof of Lemma \ref{lem:complexity-equivalence}}
\label{sec:problem-complexity-appendix}

We start with the lower bound on $\QI$. The proof is conducted by connecting set $\Pi$ to the underlying graph and selecting a specific $\delta$ that gives us a desirable lower bound.

Every vector $\bpi$ can be associated with a non-adaptive algorithm that plays each arm $i$ with probability $\pi_i$, regardless of past observations. Condition $T\sum_{i\in[N]} \pi_iG_{i,j} \ge \frac{1}{\Delta^2}$ means that, in expectation, every action in $I$ needs to be observed at least $1/\Delta^2$ times. From now on, when we talk about the number of observations of the algorithm or the number of times the algorithm plays an action, we mean the quantity in expectation. Now, we can split set $I$ into subsets

\begin{equation}
   J_{I,\bpi} \triangleq \{ j\in I: \sum_{i\in I} \pi_iG_{i,j} \ge \frac{1}{2} \sum_{i\in [N]} \pi_iG_{i,j}\} \label{eq:J-pi}
\end{equation}
and
\begin{equation}
  J_{I,\bpi}' \triangleq \{ j\in I: \sum_{i\in I} \pi_iG_{i,j} < \frac{1}{2} \sum_{i\in [N]} \pi_iG_{i,j}\} = I\setminus J_{I,\bpi}. \label{eq:J-pi-prime}
\end{equation}

Set $J_{I,\bpi}$ contains all the actions of $I$ with at least half of the observations coming from actions in $I$ while $J_{I,\bpi}'$ contains all the actions of $I$ with at least half of the observations coming from actions in $V\setminus I$

Before proceeding, we state a technical graph proposition \citep[Lemma 8]{alon2015}

\begin{proposition}
  \label{prop:graph-lemma-dominating-number}
  Let $G = (V, E)$ be a graph over $|V| = N$ vertices, and let $I \subseteq V$ be a
  set of vertices whose smallest dominating set is of size $\delta(I)$. Then, $I$ contains an independent set $U$ of size at least $ \delta(I)/ (50\log N)$, with the property that any vertex of $G$ dominates at most $\log N$ vertices of $U$
\end{proposition}

Applying Proposition \ref{prop:graph-lemma-dominating-number} to set $J_{I,\bpi}$ and a subgraph of $G$, induced by $I$, implies existence of set $U_{I,\bpi} \subseteq J_{I,\bpi}$, such that
\[
  |U_{I,\bpi}| \ge \frac{\delta^I(J_{I,\bpi})}{50\log |I|} \ge \frac{\delta^I(J_{I,\bpi})}{50\log N}
\]
with a property that every vertex of $I$ dominates at most $\log|I| \le \log N$ nodes from $U_{I,\bpi}$. Since $U_{I,\bpi}\subseteq J_{I,\bpi}$, using the definition of $J_{I,\bpi}$ together with the assumption that the number of observations of each arm in $I$ is at least $1/\Delta^2$, we know that
\[
  T\sum_{i\in I} \pi_iG_{i,j} \ge \frac{T}{2} \sum_{i\in [N]} \pi_iG_{i,j} \ge \frac{1}{2\Delta^2}
\]
for every $j\in U_{I,\bpi}$. Summing over $j\in U_{I,\bpi}$ and using the fact that each arm from $I$ can provide at most $\log N$ observations of arms in $U_{I,\bpi}$, we need the algorithm associated with $\bpi$ to play actions from $I$ at least
\[
  \frac{|U_{I,\bpi}|}{2\Delta^2\log N} \ge \frac{\delta^I(J_{I,\bpi})}{100\Delta^2\log^2 N}
\]
times to ensure enough observations in $U_{I,\bpi}$. This gives us
\begin{equation}\label{eq:lb-J}
  T\sum_{i\in I} \pi_i \ge \frac{\delta^I(J_{I,\bpi})}{100\Delta^2\log^2 N},
\end{equation}

Applying Proposition \ref{prop:graph-lemma-dominating-number} to $J_{I,\bpi}'$ and graph $G$ implies existence of set $U'_{I,\bpi} \subseteq J_{I,\bpi}'$ such that
\[
  |U'_{I,\bpi}| \ge \frac{\delta^{V}(J'_{I,\bpi})}{50\log |V|} = \frac{\delta^V(J'_{I,\bpi})}{50\log N}
\]
with a property that every vertex of $V$ dominates at most $\log|V| = \log N$ nodes from $U'_{I,\bpi}$. Since $U'_{I,\bpi}\subseteq J'_{I,\bpi}$, using the definition of $J'_{I,\bpi}$ together with the assumption that the number of observations of each arm in $I$ is at least $1/\Delta^2$, we know that
\[
  T\sum_{i\not\in I} \pi_iG_{i,j} \ge \frac{T}{2} \sum_{i\in [N]} \pi_iG_{i,j} \ge \frac{1}{2\Delta^2}
\]
for every $j\in U'_{I,\bpi}$. Summing over $j\in U'_{I,\bpi}$ and te fact that each arm from $V$ can provide at most $\log N$ observations of arms in $U'_{I,\bpi}$, we need the algorithm associated with $\bpi$ to play an action from $V$ at least

\[
  \frac{|U_{I,\bpi}|}{2\Delta^2\log N} \ge \frac{\delta^I(J_{I,\bpi})}{100\Delta^2\log^2 N}
\]
times to ensure enough observations in $U'_{I,\bpi}$. This gives us
\begin{equation}\label{eq:lb-J-prime}
  T\sum_{i\not\in I} \pi_i \ge \frac{\delta^I(J_{I,\bpi})}{100\Delta^2\log^2 N},
\end{equation}

Applying (\ref{eq:lb-J}) and (\ref{eq:lb-J-prime}) to the definition of $\QI$ gives us

\begin{align*}
  \QI = \max_\Delta\min_{\bpi\in\Pi}\min \left[T\sum_{i\in I}\pi_i\Delta + T\sum_{i\not\in I} \pi_i, T\Delta \right] & \ge \max_\Delta\min_{\bpi\in\Pi}\min\left[\frac{\delta^I(J_{I,\bpi})}{100\Delta\log^2 N} + \frac{\delta^V(J'_{I,\bpi})}{100\Delta^2\log^2 N}, T\Delta\right] \\ &\ge \max_\Delta\min_{J\subseteq I}\min\left[\frac{\delta^I(J)}{100\Delta\log^2 N} + \frac{\delta^V(I\setminus J)}{100\Delta^2\log^2 N}, T\Delta\right]
\end{align*}

Now, we are able to choose a specific value of $\Delta$ to lower bound $\QI$ further. In particular, we choose the following two values of $\Delta$

\[
  \begin{array}{lll}
    \Delta = \left( \frac{\delta^I(J)}{100T\log^2 N}\right)^{1/2}            & \implies & \QI \ge \left( \frac{T\delta^I(J)}{100\log^2 N}\right)^{1/2}              \\
    \Delta = \left( \frac{\delta^V(I\setminus J)}{100T\log^2 N}\right)^{1/3} & \implies & \QI \ge \left( \frac{T^2\delta^V(I\setminus J)}{100\log^2 N}\right)^{1/3}
  \end{array}
\]

Taking $\Delta$ which results in a larger lower bound, we obtain
\[
  \QI \ge \frac{1}{10\log N} \min_{J\subseteq I} \max \left[ \delta^I(J)^{\frac{1}{2}}T^{\frac{1}{2}},\,\delta^V(I\setminus J)^{\frac{1}{3}}T^{\frac{2}{3}}\right] = \frac{1}{10\log N} \RI.
\]

This concludes the lower bound proof. The idea of the upper bound proof is to select a specific distribution $\bpi\in\Pi$ over the set of actions and choose the optimal $\Delta$. This would give us the desirable upper bound on $\QI$.

Let $J^*\subseteq I$ be one of the minimizers in the definition of $\RI$,
\[
  J^* \in \argmin_{J\subseteq I} \max \left[ \delta^I(J)^{\frac{1}{2}}T^{\frac{1}{2}},\,\delta^V(I\setminus J)^{\frac{1}{3}}T^{\frac{2}{3}}\right].
\]
Let $D$ be a dominating set of $J^*$ from $I$ and $D'$ be a dominating set of $I\setminus J^*$ from $V$ such that $|D| = \delta^I(J^*)$ and $D' = \delta^V(I\setminus J^*)$. For a given bandit problem, with optimal arm $i^*$, we define $\bpi$ as
\[
  \pi_i = \left\{
  \begin{array}{ll}
    T^{-1}\Delta^{-2}                                & \textrm{if } i\in D \cup D' \setminus \{i^*\} \\
    1-|D \cup D' \setminus \{i^*\}|T^{-1}\Delta^{-2} & \textrm{if } i=i^*                            \\
    0                                                & \textrm{otherwise}.
  \end{array}
  \right.
\]
Distribution $\bpi$ is now a mixture of uniform distribution over $D \cup D' \setminus \{i^*\}$ with rest of the mass put on arm $i^*$. Note that for every arm $j\in I$, we know that
\[
  T\sumi \pi_i G_{i,j}\ge T\sum_{i\in D\cup D'} \pi_i G_{i,j} = T\sum_{i\in D\cup D'} \frac{1}{T\Delta^2} G_{i,j} \ge \frac{T}{T\Delta} = \frac{1}{\Delta^2},
\]
from the definition of $\bpi$ and the fact that nodes from $D$ and $D'$ dominate every node in $I$. Therefore, $\bpi$ is one of the distributions in $\Pi$. Using this $\bpi$, we can upper bound $\QI$ as
\begin{align*}
  \QI & \le \max_{\Delta}\min\left[\sum_{i\in D} \frac{T\Delta}{T\Delta^2} + \sum_{i\in D'} \frac{T}{T\Delta^2}, T\Delta \right] = \max_{\Delta}\min\left[ \frac{|D|}{\Delta} + \frac{|D'|}{\Delta^2}, T\Delta \right]                                               \\
      & = \max_{\Delta}\min\left[ \frac{\delta^I(J^*)}{\Delta} + \frac{\delta^V(I\setminus J^*)}{\Delta^2}, T\Delta \right] \le 2\max_{\Delta}\min\left[ \max\left(\frac{\delta^I(J^*)}{\Delta}, \frac{\delta^V(I\setminus J^*)}{\Delta^2} \right), T\Delta \right].
\end{align*}
The last expression is maximized for $\Delta = \max \left[ \delta^I(J^*)^{\frac{1}{2}}T^{\frac{1}{2}},\,\delta^V(I\setminus J^*)^{\frac{1}{3}}T^{\frac{2}{3}}\right]$ which implies $\QI \le 2\RI$. \qed
\section{Lower Bound Proofs}
\label{sec:lower-bound-proof}

Before we proceed with the proof of Theorem~\ref{thm:lower-bound}, we need the following information-theoretic lemma that provides a foundation for the lower-bound proof.

\begin{lemma}\label{lem:information}
  Let $G=(V,E)$ be a graph and fix a set $I \subseteq V$ with $|I| \geq 2$. Let $1/2 \geq \Delta>0$ be such that $1/\Delta^2 \leq T/2^6$. For any algorithm, there exists a problem, such that regret $R_T$ of the algorithm can be bounded as
  \begin{align*}
    R_T \geq \frac{1}{2^7} \min_{\bpi\in\Pi}\min\Bigg[ T\sum_{i\in I} \pi_i \Delta + T\sum_{i \not\in I} \pi_i,\, T\Delta \Bigg]
  \end{align*}
  where
  \begin{align*}
    \Pi & = \bigg\{\bpi\in \mathbb R^{+N}:\sumi \pi_i =1, T\sumi \pi_i G_{i, j} \geq 1/\Delta^2, \forall j \in I\bigg\}.
  \end{align*}
\end{lemma}

%%%%%%%%%%%%%%%%%%%%%%%%%%%%%%%%%%%%%%%%
%%%%%%%%%%%%%%%%%%%%%%%%%%%%%%%%%%%%%%%%
\subsection*{Proof of Lemma \ref{lem:information}}\label{sec:information}
%%%%%%%%%%%%%%%%%%%%%%%%%%%%%%%%%%%%%%%%
%%%%%%%%%%%%%%%%%%%%%%%%%%%%%%%%%%%%%%%%

Throughout the proof, we fix a set of nodes $I\subseteq V$, some node $j_0 \in I$, and any bandit policy. We use them for the rest of the proof unless said otherwise. The proof proceeds in several steps.

%%%%%%%%%%%%%%%%%%%%
%%%%%%%%%%%%%%%%%%%%
%%%%%%%%%%%%%%%%%%%%

\paragraph{Step 1: Construction of a set of difficult problems.} First, we define a set of difficult stochastic problems, indexed by $j\in I$, for a given set $I$. Then, we analyze the performance of the fixed algorithm on this set of problems. For problem $j$, the samples of each arm $i$ are independent and distributed according to $P^{(j)}_i \triangleq \cN(\mu_i^{(j)},1)$, parametrized by expectations $\mu_i^{(j)}$.

For $j_0$, we define the problem as
\[
  \mu_i^{(j_0)} \triangleq \left\{ \begin{array}{lll}
    1/2+\Delta/2 & \textrm{for} & i = j_0                  \\
    1/2          & \textrm{for} & i \in I\setminus \{j_0\} \\
    0            & \textrm{for} & i \not\in I
  \end{array} \right.
\]
For $j \in I\setminus \{j_0\}$, we define the problem as
\[
  \mu_i^{(j)} \triangleq \left\{ \begin{array}{lll}
    1/2+\Delta   & \textrm{for} & i = j                      \\
    1/2+\Delta/2 & \textrm{for} & i = j_0                    \\
    1/2          & \textrm{for} & i \in I\setminus \{j_0,j\} \\
    0            & \textrm{for} & i \not\in I
  \end{array} \right.
\]
there are several important observations to make:
\begin{itemize}
  \item Arm $j\in I$ is optimal for problem $j$.
  \item The arms for problems $j\neq j_0$ and $j_0$ differ only for arm $j$ and difference is $\Delta$.
  \item For any problem $j$, arms that are not in $I$ have a gap (distance from the optimal arm) at least $1/2$.
\end{itemize}
In what follows and for the policy we fixed at the beginning of the proof, we denoted the expectation and the probability, under the environment of bandit problem $j$, by $\E^{(j)}$ and $\P^{(j)}$.

Let $T\Ti$ denotes the number of rounds, up to round $T$, in which our fixed algorithm plays action $i$. We can associate the algorithm with a probability vector $\bpi = (\pi_1, \dots, \pi_N)$, under environment $j_0$, defined as
\begin{equation*}\label{eq:pi-def}
  \pi_i \triangleq \frac{\mathbb E^{(j_0)} [T_{T,i}]}{T} \qquad\textrm{for all}\qquad i\in [N].
\end{equation*}
Each $T\pi_i$ represents the expected number of rounds our algorithm spends playing arm $i$, under environment $j_0$.

We also write $R_T^{(j)}$ for the expected regret of the fixed policy in problem $j$. Note that
\begin{equation}\label{eq:R-j-zero-bound}
  R_T^{(j_0)} \ge \frac{T}{2}\bigg[\sum_{i\in I \setminus \{j_0\}} \pi_i \Delta + \sum_{i \not\in I} \pi_i\bigg].
\end{equation}
For the rest of the proof, we assume the existence of arm $\jbar \in I \setminus \{j_0\}$ such that
\begin{align} \label{eq:expj}
  T\sumi \pi_i G_{i,\unexploredj} \leq \frac{2}{\Delta^2}.
\end{align}

Note that the left-hand side of the previous inequality represents the expected number of observations of arm $\jbar$ using the algorithm under environment $j_0$. Markov inequality for any $j$ gives us
\[
  \Pjzero \Bigg[\sumi G_{i,j} T\Ti \geq 2^4T \sum_i G_{i,j} \pi_i\Bigg] \leq 2^{-4},
\]
which for $\jbar$, using assumption (\ref{eq:expj}), translates to
\begin{align}\label{eq:probability-of-observation}
  \Pjzero \Bigg[\sumi G_{i,\bar j} T\Ti \geq \frac{2^5}{\Delta^2}\Bigg] \leq 2^{-4}.
\end{align}
Now, let us define the event
\[
  F \triangleq  \left\{T_{T, \jbar} \leq \frac{2^5}{\Delta^2}\right\}.
\]
Note that complementary event $\bar F$ lower-bounds the number of rounds in which the algorithm played action $\unexploredj$ which is a sub-event of the event in (\ref{eq:probability-of-observation}) that lower-bounds the number of rounds in which the algorithm observed action $\unexploredj$. Therefore, we get
\begin{align}\label{eq:probaj}
  \Pjzero [ \bar F] = \Pjzero \big[ T_{T, \unexploredj} > 2^5/\Delta^2\big] \leq 2^{-4}.
\end{align}

%%%%%%%%%%%%%%%%%%%%
%%%%%%%%%%%%%%%%%%%%
%%%%%%%%%%%%%%%%%%%%

\paragraph{Step 2: Bound on the regret using KL divergence.} Note that action $j_0$ is optimal for problem $j_0$ and that action $\jbar$ is optimal for problem $\jbar$. Since choosing an action that is suboptimal leads to an instantaneous regret at least $\Delta/2$, we also have

\begin{equation}\label{eq:R-j-first-bound}
  R_T^{(\jbar)} \geq \left(T - \frac{2^5}{\Delta^2}\right)\frac{\Delta}{2}\Pjbar\left[F \right] \geq \frac{T\Delta}{4}\Pjbar\left[F \right].
\end{equation}

Then, Bretagnolle-Huber inequality (see, e.g., \citep[Theorem 14.2]{lattimore2020}) implies that
\begin{equation*}
  \Pjzero\big[ \bar F \big] + \Pjbar\left[F \right] \geq \frac{1}{2}\exp\Big(-\KL\big(\mathbb{P}^{(j_0)},\mathbb{P}^{(\bar j)} \big)\Big).
\end{equation*}
Applying Equation~\eqref{eq:probaj} implies
\begin{equation*}
  \Pjbar\left[F \right] \geq \frac{1}{2}\exp\Big(-\KL\big(\Pjzero,\Pjbar \big)\Big) - 2^{-4}.
\end{equation*}
This allows us to further bound $R_T^{(\jbar)}$ in (\ref{eq:R-j-first-bound}) as
\begin{align}\label{eq:lb_regret}
  R_T^{(\jbar)}  \geq \frac{T\Delta}{4} \left[\frac{1}{2}\exp\Big(-\KL\big(\Pjzero,\Pjbar \big)\Big) - 2^{-4}\right].
\end{align}

%%%%%%%%%%%%%%%%%%%%
%%%%%%%%%%%%%%%%%%%%
%%%%%%%%%%%%%%%%%%%%

\paragraph{Step 3: Information-theoretic bound on the number of pulls.} Because of the graph structure, the number of observations of arm $i$ at time $T$ is $N_{T,i} \triangleq \sum_{k\in[N]} G_{k,i} T_{T,k}$. So that for any problem $j$, the Kullback-Leibler divergence between $\Pj$ and $\Pjzero$ can be rewritten as follows \citep[immediate corollary of Lemma 15.1]{lattimore2020}
\[
  \KL\big(\Pjzero, \Pjbar\big) = \sumi \mathbb{E}^{(j_0)}\left[N_{T,i} \right]\KL\big(\Pjzero_{i}, \Pjbar_{i}\big).
\]
Using definition of bandit problems $j_0$ and $\jbar$, we get
\[
  \KL \big(\Pjzero, \Pjbar\big) = \Ejzero\bigg[\sumi G_{i,\bar j} T_{T,i} \bigg]\KL\big(\Pjzero_{\jbar}, \Pjbar_{\bar j}\big) = \bigg[\sumi G_{i,\bar j} \pi_i T \bigg]\frac{\Delta^2}{2}
\]
since, for any $j\in I\setminus\{j_0\}$, problems $j$ and $j_0$ differ only on action $j$ and since $\KL(\Pjzero_{j}, \Pjbar_{j}) = \Delta^2/2$. Re-ordering terms and using the assumption in Equation (\ref{eq:expj}), we get
\begin{equation}\label{eq:KLbound}
  \KL\big(\Pjzero, \Pjbar\big) = \frac{\Delta^2}{2} T\sumi G_{i,\bar j} \pi_i \le 1
\end{equation}

%%%%%%%%%%%%%%%%%%%%
%%%%%%%%%%%%%%%%%%%%
%%%%%%%%%%%%%%%%%%%%

\paragraph{Step 4: Lower bound on the regret depending on $\bpi$.} Combining Equations~\eqref{eq:KLbound} and~\eqref{eq:lb_regret}, we get
\[
  R_T^{(\jbar)} \geq \frac{T\Delta}{4} \left[\frac{1}{2} \exp\left(-1 \right) - 2^{-4}\right] \geq \frac{T\Delta}{2^6}.
\]

In case an arm $\jbar$, satisfying condition (\ref{eq:expj}), exists, regret $R^{(\jbar)}_T$ is bounded by $(T\Delta)/2^6$. In case no such arm $\jbar$ exists, i.e. $T\sum_i \pi_i G_{i, j} \geq 2/\Delta^2$ for all $j\in I\setminus \{j_0\}$, we can bound $R_T^{(j_0)}$, using (\ref{eq:R-j-zero-bound}), as

\[
  R_T^{(j_0)} \ge \frac{T}{2}\Big[\sum_{i\not\in I \setminus \{j_0\}} \pi_i \Delta + \sum_{i \not\in I} \pi_i\Big].
\]

Therefore, we have that for any policy, it holds that

\[
  R^* \geq \max_{j \in I} R_T^{(j)} \geq \min_{\pi\in \R^{+N}:\sum_i \pi_i =1, T\sum_i \pi_i G_{i, j} \geq 2/\Delta^2, \forall j \in I\setminus \{j_0\}}  \min\Bigg[\frac{T}{2}\Big[ \sum_{i\in I \setminus \{j_0\}} \pi_i \Delta + \sum_{i \not\in I} \pi_i\Big], \frac{T\Delta}{2^6} \Bigg].
\]

We have proved this inequality for any $j_0 \in I$ so that if we take two $j_0',j_0'' \in I:j_0' \neq j_0''$, we would get:

\[
  R^* \geq \max_{j_0 \in \{j_0', j_0''\}} \min_{\pi\in \mathbb R^{+N}:\sum_i \pi_i =1, T\sum_i \pi_i G_{i, j} \geq 2/\Delta^2, \forall j \in I\setminus \{j_0\}} \min\Bigg[\frac{T}{2}\Big[ \sum_{i\in I \setminus \{j_0\}} \pi_i \Delta + \sum_{i \not\in I} \pi_i\Big], \frac{T\Delta}{2^6} \Bigg].
\]
Since for any function $A$ of $\pi$ we have $\max(\min_{\pi: \pi_i > a_i} A(\pi), \min_{\pi: \pi_i > b_i} A(\pi)) \geq \min_{\pi: \pi_i > (a_i+b_i)/2} A(\pi)/2$ this implies the result of the lemma, namely
\[
  R^* \geq \frac{1}{2^7} \min_{\pi\in \mathbb R^{+N}:\sum_i \pi_i =1, T\sum_i \pi_i G_{i, j} \geq 1/\Delta^2, \forall j \in I} \min\Bigg[T \sum_{i\in I} \pi_i \Delta + T\sum_{i \not\in I} \pi_i, T\Delta \Bigg].
\]

\subsection{Proof of Theorem~\ref{thm:lower-bound}}
The starting point of the proof is the statement of Lemma~\ref{lem:information}. From the definition of the problem complexity $\Qstar$, we can directly show that for any $I\subset V$
\[
R_T \ge \frac{\QI}{2^7}.
\]
Since this inequality holds for any $I\subset V$, it holds also for $I$ that maximizes $\QI$. This implies that
\begin{equation}
R_T \ge \max_{I\subset V}\frac{\QI}{2^7} = \frac{\Qstar}{2^7}. \label{eq:Q-lower-bound}    
\end{equation}
Applying Lemma~\ref{lem:complexity-equivalence} to (\ref{eq:Q-lower-bound}), we have
\[
R_T \ge \frac{\Qstar}{2^7} \ge \frac{\Rstar}{2^710\log N}. \qed
\]

\section{Efficient Construction of the Proxies $\J\tkl$, $\J'\tkl$ and Dominating Sets $\D\tkl, \D'\tkl$ from Corollary~\ref{cor:partition-splitting}}
\label{sec:convex-program}

In this section, we construct the partition of the set $I\tkl$ through sets $\J\tkl$, $\J'\tkl$ and the associated dominating sets $\D\tkl$, $\D'\tkl$ from Corollary~\ref{cor:partition-splitting}.

In order to do this, we provide a construction for an arbitrary set $I$, and can then apply this procedure to the $I\tkl$. Consider therefore any arbitrary set $I \subset V$.

For any $1\geq \Delta >0$, consider the optimisation problem $\QIDelta$ from Definition~\ref{def:problem-complexity-Q}, namely:
  \begin{align}\label{eq:opti}
    \QIDelta \triangleq \min_{\bpi\in\Pi} Q_{I,\Delta} (\pi) \triangleq \min_{\bpi\in\Pi}\min\Bigg[ T\sum_{i\in I} \pi_i \Delta + T\sum_{i \not\in I} \pi_i,\, T\Delta \Bigg]
    \end{align}
  where 
  \begin{align*}
    \Pi = \Pi^\Delta & = \bigg\{\bpi\in \mathbb R_{+}^{N}:\sumi \pi_i \leq 1,\ T\sumi \pi_i G_{i, j} \geq 1/\Delta^2, \forall j \in I\bigg\}.
  \end{align*}
  This is a linear optimization problem under linear constraints. If the set of solutions is not empty, we know e.g.~by using Algorithm \textit{Main} from~\citep{cohen2019} with precision $\Delta^2/(N^4T^2)$ that there there is a solver for such problem doing less than $N^3 \log(T^2N^5/\Delta^2)$ basic operations and that would output a solution $\pi^\Delta$ such that $\pi^\Delta \in \Pi$ and
  \begin{equation}\label{eq:optimQ}
      \QIDelta \leq Q_{I,\Delta}(\pi^\Delta) \leq \QIDelta+1.
  \end{equation}
  See Theorem 1 from~\citep{cohen2019} for more details. From there, we use $\pi^\Delta$, which is the solution of algorithm outputted by Algorithm \textit{Main} from~\citep{cohen2019} with precision $\Delta^2/(N^4T^2)$ on the optimization problem from above (characterized by $\Delta$).

Based on this, the procedure that we use is described in Algorithm~\ref{alg:proxy}. We consider a logarithmic grid of $\Delta$ of size $\lfloor\log(NT)\rfloor +2$, and associated to the policies $\bar \pi^\Delta$ constructed on this grid, we construct sets $\bar J^\Delta = J_{I, \pi^\Delta}$ as in (\ref{eq:J-pi}), and the associated complement $\bar J^{',\Delta}$, and then compute a greedy approximation of their dominating sets as described in Algorithm~\ref{alg:greedy-domination-set}. The description and discussion of Algorithm~\ref{alg:greedy-domination-set} is postponed to Subsection~\ref{sec:dominating-algorithm}, as it is a very standard result. We finally take the stationary policy $\bar \pi$ corresponding to the best possible case $\pi^\Delta$ for a criterion that resembles the one of Definition~\ref{def:problem-complexity-R}, and output the associated sets $\bar J, \bar J', \bar D, \bar D'$. 

\begin{algorithm}[tb]
  \caption{\textsc{Algorithm outputting proxies of $ J, J'$ and of their dominating sets}}
  \label{alg:proxy}
  \begin{algorithmic}
    \STATE {\bfseries Input:} $G = (V,E)$, set $I \subset V$.
    \FOR{ $\Delta \in \{2^{-1}, 2^{-2}, ..., 2^{-(\lfloor\log(NT)\rfloor +1)}\}$}
    \STATE Compute $\bar \pi^\Delta$ as the output of Algorithm \textit{Main} from~\citep{cohen2019} applied on Problem~\ref{eq:opti} with precision $\Delta^2/(N^4T^2)$
        \STATE Set $\bar J^\Delta = J_{I,\pi^\Delta}$ \qquad(as defined in Equation (\ref{eq:J-pi}))
    \STATE Set $\bar J^{',\Delta} = I \setminus J^\Delta$
    \STATE Set $\bar D^\Delta$ as the output of Algorithm~\ref{alg:greedy-domination-set} on $\bar J^\Delta$
    \STATE Set $\bar D^{',\Delta}$ as the output of Algorithm~\ref{alg:greedy-domination-set} on $\bar J^{',\Delta}$
    \ENDFOR
    \STATE Set  $\bar \Delta = \argmin_\Delta\{ \sqrt{\bar D^{\Delta}T} \land [(\bar D^{',\Delta})^{1/3}T^{2/3}] \}$
    \STATE Set $\bar \pi = \pi^{\bar \Delta}, \bar J = \bar J^{\bar \Delta}, \bar J' = \bar J^{',\bar \Delta}, \bar D = \bar D^{\bar \Delta}, \bar D' = \bar D^{',\bar \Delta}$
  \STATE {\bfseries Output:} $\bar J, \bar J',\bar D, \bar D'$
  \end{algorithmic}
\end{algorithm}

The following lemma holds for Algorithm~\ref{alg:proxy}.
\begin{lemma}
\label{lem:blackbox}
Let $G=(V, E)$ be a graph and $I\subset V$ be any subset of vertices. Then Algorithm~\ref{alg:proxy} applied to $I$ runs in less than $30 N^3 \log(TN)^2$ iterations,
and produces sets $\J\subset I$, $\J'\triangleq I\setminus \J$, $\D\subset I$, and $\D'\subset V$ such that $\D$ dominates set $\J$, $\D'$ dominates $\bar J'$,
\[
\max\left\{ \delta^I(\J)^\frac{1}{2}T^\frac{1}{2},\, \delta^V(\J')^\frac{1}{3}T^\frac{2}{3} \right\} \le \blackboxconstant\RI,
\]

\alex{blackboxconstant should in fact be $96 \times 10^4 \log^5 N$}
and
\[
|\D| \le \delta^I(\bar J) \log(|I|), \qquad |\D'| \le \delta^V(\bar J')\log(|V|).
\]
\end{lemma}
Corollary~\ref{cor:partition-splitting} is a direct application of this lemma.

\subsection{Proof of Lemma~\ref{lem:blackbox}}

Let $\Delta \in \{2^{-1}, 2^{-2}, ..., 2^{-(\lfloor\log(NT)\rfloor +1)}\}$. Note fist that for any $1/2\geq \Delta' >0$ such that $ \Delta'/2 \leq \Delta \leq \Delta'$, we have for any policy $\pi$ that

\[
  Q_{I, \Delta'} (\pi)/2 \leq Q_{I,\Delta} (\pi) \leq Q_{I, \Delta'} (\pi),
\]

by definition of $Q_{I,\Delta}(.)$. Moreover for $\Delta' \leq 2^{-(\lfloor\log(NT)\rfloor +1)}$, we have that for any stationary policy $\pi$:
\[
  Q_{I,\Delta'} (\pi) \leq T\Delta' \leq 1/N.
\]
So that if, for any $\Delta'$, we write $G(\Delta')$ for the projection of $\Delta'$ on the closest larger element of $\Delta \in \{2^{-1}, 2^{-2}, ..., 2^{-(\lfloor\log(NT)\rfloor +1)}\}$, we have 
\[
  Q_{I, \Delta'} (\pi)/2 \leq Q_{I,G(\Delta')} (\pi) + 1/N \leq Q_{I, \Delta'} (\pi) + 2/N.
\]

This first implies taking $\pi = \pi^{G(\Delta')}$ and since $\Pi^{G(\Delta')} \in 2\Pi^{\Delta'} $ (where $2\Pi^{\Delta'}$ is the set $\{x:x/2 \in 2\Pi^{\Delta'}\}$)

\[
Q_{I, \Delta'}^\star/2 \leq Q_{I, G(\Delta')}^\star + 1/N.
\]

This also implies by taking the minimum over $\pi \in \Pi^{\Delta'}$

\[
  Q_{I, \Delta'}^\star \leq \min_{\pi \in \Pi^{\Delta'}}Q_{I,G(\Delta')} (\pi) + 1/N \leq 2Q_{I, \Delta'}^\star + 2/N.
\]
And so since $\Pi^{\Delta'} \subset \Pi^{G(\Delta')}$ we have 

\[
  Q_{I,G(\Delta')}^\star + 1/N \leq Q_{I, \Delta'}^\star + 2/N.
\]

So that in the end for any $\Delta'$

\[
    Q_{I, \Delta'}^\star/2 \leq Q_{I,G(\Delta')}^\star  \leq 2Q_{I, \Delta'}^\star + 1/N.
\]

So that by Equation~\eqref{eq:optimQ}

\begin{equation*}
Q_{I, \Delta'}^\star/2 \leq Q_{I,G(\Delta')}(\bar \pi^{G(\Delta')})  \leq 2Q_{I, \Delta'}^\star + 1/N +1.
\end{equation*}

Set $\Delta = G(\Delta')$. Following exactly the same steps as in the proof of Proposition~\ref{lem:complexity-equivalence} in Section~\ref{sec:problem-complexity-appendix}, we know that for $J_{I, \bar \pi^{\Delta}}$ defined as in Equation (\ref{eq:J-pi}) we have 

\begin{align*}
   Q_{I,\Delta}(\bar \pi^{\Delta})   &\ge    \frac{1}{100 \log^2 N}\min\left[\frac{\delta^I(J_{I, \bar \pi^{\Delta}})}{\Delta} + \frac{\delta^V(I\setminus J_{I, \bar \pi^{\Delta}})}{\Delta^2}, T\Delta\right].
\end{align*}

We remind that $J,J' = I\setminus J$ minimise the equation in Definition~\eqref{def:problem-complexity-R}. Consider now $\Delta^\star = G(400\log^2 N\Big[\sqrt{\frac{\delta^I(J)}{T}} \lor \left(\frac{\delta^V(J')}{T}\right)^{1/3}\Big])$. For $\Delta^\star$ in particular, we have from the previous equation and from Equation~\eqref{eq:optimQ}
\begin{align*}
2Q_{I, \Delta^\star}^\star  +1+1/N  \ge Q_{I,\Delta^\star}(\bar \pi^{\Delta^\star})  &\ge    \frac{1}{100 \log^2 N}\min\left[\frac{\delta^I(J_{I, \bar \pi^{\Delta^\star}})}{\Delta^\star} + \frac{\delta^V(I\setminus J_{I, \bar\pi^{\Delta^\star}})}{\Delta^{\star 2}}, T\Delta^\star\right].
\end{align*}
So that from Proposition~\ref{lem:complexity-equivalence} and the last equation
\begin{align*}
   4 \max\left[\sqrt{\delta^I(J) T} , (\delta^V(J))^{1/3}T^{2/3}\right] &\ge 4\RI \ge 2\QI \ge 2Q_{I, \Delta^\star}^\star \\ &\ge    \frac{1}{100 \log^2 N}\min\left[\frac{\delta^I(J_{I, \bar \pi^{\Delta^\star}})}{\Delta^\star} + \frac{\delta^V(I\setminus J_{I, \bar \pi^{\Delta^\star}})}{\Delta^{\star 2}}, T\Delta^\star\right] -  1- 1/N .
\end{align*}
Since by definition of $\Delta^\star$ we have that $\frac{1}{100 \log^2 N}T\Delta^\star-1-1/N > 4 \max\left[\sqrt{J^{\star} T} , (J^{'\star})^{1/3}T^{2/3}\right]$ this implies
\begin{align*}
   4 \max\left[\sqrt{\delta^I(J) T} , (\delta^V(J'))^{1/3}T^{2/3}\right] \ge  \frac{1}{100 \log^2 N} \left[\frac{\delta^I(J_{I, \bar \pi^{\Delta^\star}})}{\Delta^\star} + \frac{\delta^V(I\setminus J_{I, \bar \pi^{\Delta^\star}})}{\Delta^{\star 2}}\right] -  1- 1/N .
\end{align*}
By definition of $\Delta^\star$, we have
\begin{align*}
   4 \max\left[\sqrt{\delta^I(J) T} , (\delta^V(J'))^{1/3}T^{2/3}\right] \ge  \frac{1}{8\times 10^4 \log^4 N} \left[\frac{\delta^I(J_{I, \bar \pi^{\Delta^\star}})}{\sqrt{\delta^I(J)}} \sqrt{T}+ \frac{\delta^V(I\setminus J_{I, \bar \pi^{\Delta^\star}})}{\delta^V(J')^{2/3}} T^{2/3}\right] -  1- 1/N ,
\end{align*}
i.e.
\begin{align*}
   48 \times 10^4 \log^4 N \max\left[\sqrt{\delta^I(J) T} , (\delta^V(J'))^{1/3}T^{2/3}\right] \ge   \max\left[\frac{\delta^I(J_{I, \bar \pi^{\Delta^\star}})}{\sqrt{\delta^I(J)}} \sqrt{T}, \frac{\delta^V(I\setminus J_{I, \bar \pi^{\Delta^\star}})}{\delta^V(J')^{2/3}} T^{2/3}\right].
\end{align*}
So that in the end
\begin{align*}
   48 \times 10^4 \log^4 N \max\left[\sqrt{\delta^I(J) T} , (\delta^V(J'))^{1/3}T^{2/3}\right] \ge   \max\left[\sqrt{\delta^I(J_{I, \bar \pi^{\Delta^\star}})T}, \delta^V(I\setminus J_{I, \bar \pi^{\Delta^\star}})^{1/3} T^{2/3}\right].
\end{align*}
By Subsection~\ref{sec:dominating-algorithm}, we therefore have
\begin{align*}
   48 \times 10^4 \log^5 N \max\left[\sqrt{\delta^I(J) T} , (\delta^V(J'))^{1/3}T^{2/3}\right] \ge   \max\left[\sqrt{\bar D_{I, \bar \pi^{\Delta^\star}}T}, (\bar D_{I, \bar \pi^{\Delta^\star}}')^{1/3} T^{2/3}\right].
\end{align*}
Since by definition of the algorithm we have $\max\left[\sqrt{\bar D_{I, \bar \pi^{\Delta^\star}}T}, (\bar D_{I, \bar \pi^{\Delta^\star}}')^{1/3} T^{2/3}\right] \geq \max\left[\sqrt{\bar DT}, (\bar D')^{1/3} T^{2/3}\right]$, this concludes the proof.

\subsection{Greedy Algorithm for Dominating Set}
\label{sec:dominating-algorithm}

Finding the smallest dominating set is an NP-hard problem, however, a simple greedy algorithm can find an approximate solution, only a logarithmic factor away from the optimal solution. The algorithm is described in Algorithm~\ref{alg:greedy-domination-set} and the theoretical guarantees can be found in Theorem~\ref{thm:greedy-algorithm}.

\begin{algorithm}[tb]
  \caption{\textsc{Greedy algorithm for dominating set}}
  \label{alg:greedy-domination-set}
  \begin{algorithmic}
    \STATE {\bfseries Input:}
    \STATE $G = (V,E)$, sets $A,B\subset V$ such that $A$ dominates $B$.
    \STATE $D = \emptyset$
    \REPEAT
    \STATE $d = \argmax_{v\in A} |\Nout_v \cap B|$ \qquad(ties resolved arbitrarily)
    \STATE $D = D \cup {d}$
    \STATE $B = B \setminus \Nout_d$
    \UNTIL{|B| = 0 }
    \STATE {\bfseries Output:} D
  \end{algorithmic}
\end{algorithm}

\begin{theorem}
\label{thm:greedy-algorithm}
Let $G = (V, E)$ be a graph with two sets of vertices $A, B\subset V$ such that $A$ dominates $B$. Then Algorithm~\ref{alg:greedy-domination-set} produces set $D\subset A$ that dominates $B$ from $A$, such that
\[|D| \le \log(N) \delta^A(B).\]
Moreover, the computational complexity of Algorithm~\ref{alg:greedy-domination-set} is at most linear in the number of vertices.
\end{theorem}
This is a standard result that can be found for example in \citep{chvatal1979}.

\section{Proof of Theorem \ref{thm:upper-bound}}
\label{sec:upper-bound-proof}
  We start the proof with a standard proposition that can be used as the first step in the analysis of most of the algorithms based on \textsc{Exp3}

\begin{proposition}
  \label{prop:standard-upper-bound}
  Let everything be defined as in Algorithm \ref{alg:main}. Then we have
  \[
     \EE{\sumt\sumi q\ti\ell\ti - \min_{k\in[N]} \sumt \ell_{t,k}} \EE{\frac{\log N}{\eta_{T+1}}} + \EE{\sumt\frac{\etat}{2}\sumi\frac{q\ti}{P\ti}}.
  \]
\end{proposition}

\begin{proof}
The proof of this proposition is based on the proof by \cite{gyorfi2007}. First, lets define $W'_{t+1}$, similarly to $W_{t+1}$, using learning rate from round $t$ instead of $t+1$.
\[
  W'_{t+1} = \frac{1}{N}\sumi\exp(-\etat\hLoss_{t+1,i}).
\]
Following standard analysis of \textsc{Exp3} algorithms with adaptive learning rate (e.g. \citep{kocak2014}), we can obtain
\begin{align*}
  \frac{1}{\etat}\log\frac{W'_{t+1}}{W_t}
   & = \frac{1}{\etat} \log\sumi\frac{ (1/N)\exp(-\etat\hLoss\ti) }{W_t}                                  \\
   & = \frac{1}{\etat}\log\sumi\frac{w\ti\exp(-\etat\hl\ti)}{W_t}                                         \\
   & = \frac{1}{\etat}\log\sumi q\ti\exp(-\etat\hl\ti)                                                    \\
   & \le \frac{1}{\etat}\log\sumi q\ti\left( 1-\etat\hl\ti + \frac{1}{2}(\etat\hl\ti)^2\right)            \\
   & = \frac{1}{\etat}\log\left( 1 - \etat\sumi q\ti\hl\ti + \frac{\etat^2}{2}\sumi q\ti(\hl\ti)^2\right) \\
   & \le -\sumi q\ti\hl\ti + \frac{\etat}{2}\sumi q\ti(\hl\ti)^2,
\end{align*}
where we used inequality $\exp(-x) \le 1-x+x^2/2$ for $x\ge0$ as well as inequality $\log(1-x) \le -x$ that holds for any $x$. Rearranging the terms in the previous inequality, we obtain
\begin{align}
  \sumi q\ti\hl\ti & \le
  \left[\left( \frac{\log W_t}{\etat} - \frac{\log W_{t+1}}{\eta_{t+1}}\right) + \left(  \frac{\log W_{t+1}}{\eta_{t+1}} - \frac{\log W'_{t+1}}{\etat}\right)\right] + \frac{\etat}{2}\sumi q\ti(\hl\ti)^2 \label{eq:removingWprime}
\end{align}
The second term in the brackets can be further bounded using
\[
  W_{t+1} = \sumi\exp(-\eta_{t+1}\hLoss\ti) = \sumi\exp(-\etat\hLoss\ti)^{\frac{\eta_{t+1}}{\etat}} \le \left(\sumi\exp(-\etat\hLoss\ti)\right)^{\frac{\eta_{t+1}}{\etat}} = (W'_{t+1})^{\frac{\eta_{t+1}}{\etat}},
\]
where we applied Jensen's inequality to the concave function $x^{\frac{\eta_{t+1}}{\etat}}$ thanks to the assumption that $\eta_{t+1} \le \etat$. Therefore, we obtain
\[
  \left(  \frac{\log W_{t+1}}{\eta_{t+1}} - \frac{\log W'_{t+1}}{\etat}\right)\le 0,
\]
which further simplifies (\ref{eq:removingWprime}) to obtain
\[
  \sumi q\ti\hl\ti \le
  \left( \frac{\log W_t}{\etat} - \frac{\log W_{t+1}}{\eta_{t+1}}\right) + \frac{\etat}{2}\sumi q\ti(\hl\ti)^2
\]

The next step is summing over time and taking expectation
\[
  \EE{\sumt\sumi q\ti\ell\ti} \le \EE{\frac{\log W_{T+1}}{\eta_{T+1}}} + \EE{\sumt\frac{\etat}{2}\sumi \frac{q\ti}{P\ti}\ell\ti^2}.
\]
Lower-bounding $W_{T+1}$ by $\max_{k\in[N]} w_{T+1,k}$ and using the definition of exponential weights, we obtain
\[
  \EE{\sumt\sumi q\ti\ell\ti - \min_{k\in[N]} \sumt \ell_{t,k}} \le \EE{\frac{\log N}{\eta_{T+1}}} + \EE{\sumt\frac{\etat}{2}\sumi\frac{q\ti}{P\ti}}.
\]
\end{proof}

Using the definition of $p\ti$ from Equation (\ref{eq:mixed-distribution}), we can lower bound the left-hand side of the inequality in Proposition \ref{prop:standard-upper-bound} as
\[  \EE{\sumt\sumi q\ti\ell\ti - \min_{k\in[N]} \sumt \ell_{t,k}} \ge R_T - 2\EE{\sumt\gammat}.
\]
  This changes the inequality from Proposition \ref{prop:standard-upper-bound} to
  \begin{equation}
    R_T \le\EE{ 2\sumt\gammat + \frac{\log N}{\eta_{T+1}} + \sumt\frac{\etat}{2}\sumi\frac{q\ti}{P\ti}}. \label{eq:regret-upper-bound}
  \end{equation}
  The next step is to upper bound the last term inside of the expectation, starting with the inner most sum. We can split this sum into several parts, depending on the partition in which the arm is situated and upper bound each par separately.
  \begin{align}
    \sumi\frac{q\ti}{P\ti} & = \sumk\suml\sumJ \frac{q\ti}{P\ti}       \label{ex:part-J}      \\
                           & + \sumk\suml\sumJprime \frac{q\ti}{P\ti}  \label{ex:part-Jprime} \\
                           & + \sum_{i\in I_{t,K+1}} \frac{q\ti}{P\ti} \label{ex:part-I}
  \end{align}
  \textbf{Bounding expression (\ref{ex:part-J}).} Partition $\J\tkl$ contains only arms $i$ with $q\ti \in (2^{-k}, 2^{-k+1}]$ and $\deg\tk(i) \in (N2^{-l}, N2^{-l+1}]$. Therefore, $q\ti$ can be simply upper bounded by the largest possible value of $q\ti$, which is $2^{-k+1}$. We also know that $i\in \J\tkl$ has at least $N2^{-l}$ neighbors in $I\tk$ each of which has the probability of being played $p\tj$ which is at least $q\tj/2$ from the definition of $p\tj$ and the fact that $\gamma_t\le 1/2$. Now, we lower bound $P\ti$ by the smallest possible number of neighbors in $I\tk$ and the smallest possible value of $q\tj$ for neighbors $j$ of $i$ to obtain $P\ti \ge N2^{-l} 2^{-k} 2^{-1}$. This gives us the following upper bound on (\ref{ex:part-J})
  \[
    (\ref{ex:part-J}) \le \sumk\suml \frac{|\J\tkl|2^{-k+1}}{N2^{-l} 2^{-k} 2^{-1}} \le \sumk\suml \frac{|\J\tkl|2^{3}}{N2^{-l+1}} \le \sumk\suml2^3|\D\tkl|.
  \]
  The last inequality holds thanks to the fact that $|\D\tkl|$ is the size of the dominating set of $\J\tkl$ with every vertex dominating at most $N2^{-l+1}$ other nodes. This means that the number of possible edges $|\D\tkl| N2^{-l+1}$ from $\D\tkl$ to $\J\tkl$ needs to be larger than the number of vertices in $\J\tkl$.

  \textbf{Bounding expression (\ref{ex:part-Jprime}).} Since $\D'\tkl$ is a dominating set of $\J'\tkl$, we know that every vertex of $\J'\tkl$ is dominated by at least one vertex from the dominating set $\D'\tkl$. For each vertex of this dominating set, we used a mixing in Definition~\ref{def:exploration-distribution}, and therefore we observe every vertex of $\D'\tkl$ with probability at least $\gammat / ((KL+1)|\D'\tkl|)$. As a consequence, we have that for every $i\in J'\tkl$, probability of observing $i$ is at least $\gammat / ((KL+1)|D'\tkl|)$. This gives us the following upper bound on (\ref{ex:part-Jprime})
  
\[
    (\ref{ex:part-Jprime}) \le \sumk\suml \frac{(KL+1)|\D'\tkl|}{\gammat}
\]

  \textbf{Bounding expression (\ref{ex:part-I}).} We know that every arm $i$ from $I_{t,K+1}$ is such that $q\ti$, by definition, is smaller than $1/N^5$. From the definition of exploration distribution in Definition~\ref{def:exploration-distribution}, we know that the corresponding $P\ti$ is lower bounded by $\gammat / ((KL+1)N)$. This gives us the following upper bound on (\ref{ex:part-I})
  \begin{align*}
    (\ref{ex:part-I}) \le \sum_{i\in I_{t,K+1}} \frac{(KL+1)N}{\gammat N^5} \le \frac{1}{\gammat}.
  \end{align*}
  The last inequality holds thanks to the fact that $(KL + 1) \le N^2$, for any positive integer $N$, from the definition of $K$ and $L$. Note that in particular
  \begin{align*}
    (\ref{ex:part-I}) \le \sumk\suml \frac{(KL+1)\log(N)|\D'\tkl|}{\gammat},
  \end{align*}
  i.e.~we can upper bound the last term using the same expression as in the upper bound of (\ref{ex:part-Jprime}).

  Using all three bounds together, we obtain the following upper bound for $\frac{\etat}{2}\sumi \frac{q\ti}{P\ti}$
  \begin{equation}
     \frac{\etat}{2}\sumi \frac{q\ti}{P\ti} \le \frac{\etat}{2}\sumkl\left( 8|\D\tkl| + \frac{2(KL+1)\log(N)|\D'\tkl|}{\gammat}\right). \label{eq:q-over-P-bound}
  \end{equation}

  Now we can define an auxiliary learning rate $\eta\tkl$ for each individual summand as
  \begin{align*}
    \eta\tkl & \triangleq \min\left(|\D\tkl|^{-\frac{1}{2}} T^{-\frac{1}{2}} , (|\D'\tkl|)^{-\frac{1}{3}} T^{-\frac{2}{3}} \right)
  \end{align*}
  Note that $\etat$ is defined as $\min_{s\in[t],k\in[K], l\in[L]} \eta_{s,k,l}$ and therefore we can upper bound it  by $\eta\tkl$ in (\ref{eq:q-over-P-bound}) to obtain
  \begin{align*}
    \frac{\etat}{2}\sumi \frac{q\ti}{P\ti}
     & \le \sumkl \left( 4\etat|\D\tkl| + \etat^2T(C-4)|\D'\tkl|\right)                                                 \\
     & \le \sumkl \left( 4\eta\tkl|\D\tkl| + \eta\tkl^2T(C-4)|\D'\tkl|\right)                                           \\
     & = C\sumkl \max\left(|\D\tkl|^{\frac{1}{2}} T^{-\frac{1}{2}} , |\D'\tkl|^{\frac{1}{3}} T^{-\frac{1}{3}} \right) \\
     & \le C\sumkl \blackboxconstant\log N\frac{\Rstar}{T}
  \end{align*}
  for $C = 4 + (KL+1)\log(N)$. Summing over time, we obtain the following upper bound
  \begin{equation}
    \sumt\frac{\etat}{2}\sumi \frac{q\ti}{P\ti} \le C\sumkl \max\left(\delta_*^{\frac{1}{2}} T^{\frac{1}{2}} , {\delta'}_*^{\frac{1}{3}} T^{\frac{2}{3}} \right) \label{eq:bound-A}
  \end{equation}

  Now we are able to bound the last term in (\ref{eq:regret-upper-bound}). The next step is bounding the first two terms to obtain the regret upper bound. In order to do so, we can use the definitions of $\gammat$ and $\etat$ and the fact that $\{\etat\}_{t\in[T]}$ is a non-increasing sequence to obtain
  \begin{align}
    2\sumt \gammat + \frac{\log N}{\eta_{T+1}} &= 2\sumt\min\left\{\frac{1}{T\etat},\frac{1}{2}\right\} + \frac{\log N}{\eta_{T+1}} \nonumber                                                                                        \\
                   & \le 2\sumt\frac{1}{T\etat} + \frac{\log N}{\eta_{T+1}} \nonumber                                                                 \\
                   & \le 2\sumt \frac{1}{T\eta_{T+1}} + \frac{\log N}{\eta_{T+1}} = \frac{2 + \log N}{\eta_{T+1}} \nonumber                           \\
                   & \le \blackboxconstant\log N(2+\log N)\Rstar \label{eq:bound-B}
  \end{align}
  The proof is concluded by applying bounds (\ref{eq:bound-A}) and (\ref{eq:bound-B}) to expression (\ref{eq:regret-upper-bound})

\section{Discussion}
This section covers the proofs of corollaries in Section~\ref{sec:discussion}.

\subsection{Proof of Corollary~\ref{cor:rate-for-large-T}}
\label{sec:proof-rate-for-large-T}
The starting point of this proof is the Definition~\ref{def:problem-complexity-R}. The idea is to show that for $T\ge\alpha^3$, it is optimal to take $J=I$, in the minimization part of the problem complexity definition, regardless of set $I$.

Let us fix $I$ and assume that the optimization problem is maximized for some $J\not= I$. This also means that $I\setminus J$ is a non-empty set and therefore, needs to be dominated by at least one node, i.e. $\delta^V(I\setminus J) \ge 1$. Note that the independence number of a graph is always an upper bound on the dominating number since the largest independent set is connected to all the vertices in the graph. This means that the independence number $\alpha$ of $G$ can be lower bounded by the independence number $\delta^I(I)$ of the graph induced by $I$ which in turn can be lower bounded by $\delta^I(J)$, for any $J\subseteq I$. Using these observations, together with assumption $T \ge \alpha^3$, we get
\begin{align*}
    \delta^V(I\setminus J)^\frac{1}{3}T^{\frac{2}{3}} \ge T^\frac{2}{3} \ge {\alpha^\frac{1}{2}T^\frac{1}{2}} \ge {\delta^V(V)^\frac{1}{2}T^\frac{1}{2}} \ge {\delta^I(I)^\frac{1}{2}T^\frac{1}{2}} \ge {\delta^I(J)^\frac{1}{2}T^\frac{1}{2}}.
\end{align*}
However, setting $J=I$ would decrease $\delta^V(I\setminus J)^\frac{1}{3}T^{\frac{2}{3}}$ to 0 and therefore, improve the minimization problem from the problem complexity definition. Now that we know that the optimal $J$ is equal to $I$, we are ready to find $I$ that maximizes the problem complexity. Let $I$ be the largest independent set. The size of $I$ is now $\alpha$ and the only way to dominate $J = I$ using only nodes from $I$ is by using all the nodes, therefore, $\delta^I(J) = |I| = \alpha$ and the problem complexity $\Rstar = \max(0, \sqrt{\alpha T})$.

\subsection{Proof of Corollary~\ref{cor:rate-for-star-graph}}
\label{sec:proof-rate-for-small-T}

The starting point of this proof is the Definition~\ref{def:problem-complexity-R}. We know that the graph contains one vertex connected to all other nodes. Therefore, $\delta^V(I\setminus J)$ is either 0, if $I\setminus J$ is an empty set, or 1, if $I\setminus J$ is non-empty.

In case $\delta^V(I\setminus J) = 1$, the minimization part of the problem complexity suggests that $J$ can be empty set since it does not change the value of $\delta^V(I\setminus J)$ while reducing $\delta^I(J)$ to 0, regardless of the choice of $I$.

In case $\delta^V(I\setminus J) = 0$, we know that $I = J$ and therefore, the value $\delta^I(J)$ can be upperbounded by $\alpha$ - case when $I$ contains all $N-1 = \alpha$ independent vertices.

Since only these two cases are possible, we have a freedom to choose $J$ for which the value is smaller. Therefore, the problem complexity $\Rstar$ can be computed as
\[
\Rstar = \min \left\{  \alpha^{\frac{1}{2}}T^{\frac{1}{2}},\,1^{\frac{1}{3}}T^{\frac{2}{3}}\right\}.
\]
Using assumption $T<\alpha^3$ gives us
\(
\alpha^{\frac{1}{2}}T^{\frac{1}{2}} > 1^{\frac{1}{3}}T^{\frac{2}{3}}
\)
and therefore $\Rstar = T^{\frac{2}{3}}$, which concludes the proof.

\end{document}